\documentclass[letterpaper,journal]{IEEEtran}
\usepackage{amsmath,amsfonts,amssymb}
\usepackage{algorithmic}
\usepackage{algorithm}
\usepackage[linesnumbered,ruled,vlined,algo2e]{algorithm2e}
\usepackage{array}
\usepackage[caption=false,font=footnotesize]{subfig}
\usepackage{textcomp}
\usepackage{stfloats}
\usepackage{url}
\usepackage{verbatim}
\usepackage{graphicx}
\usepackage[space]{grffile}
\usepackage[numbers,sort&compress]{natbib}
\usepackage{booktabs}
\usepackage{multirow} 
\usepackage{tabularx}
\usepackage{array}
\usepackage{multirow}
\graphicspath{{./graphs/}}

\newcommand{\vdctx}{\boldsymbol{\psi}}
\newcommand{\vstate}{\mathbf{x}}
\newcommand{\vact}{\mathbf{a}}
\newcommand{\vobs}{\mathbf{o}}
\newcommand{\vhid}{\mathbf{h}}
\newcommand{\vlat}{\mathbf{z}}
\newcommand{\vparams}{{\boldsymbol{\theta}^{\mathrm{v}}}}
\newcommand{\vphys}{\boldsymbol{\Phi}}

\DeclareMathOperator{\sg}{sg}
\DeclareMathOperator{\dynpred}{DynPred}
\DeclareMathOperator{\Embed}{Embed}
\DeclareMathOperator{\KL}{KL}
\DeclareMathOperator{\TireF}{TireF}
\DeclareMathOperator{\ODE}{ODE}
\DeclareMathOperator{\Enc}{Enc}

\DeclareMathOperator{\LN}{LN}
\DeclareMathOperator{\Block}{Block}
\DeclareMathOperator{\RoPE}{RoPE}
\DeclareMathOperator{\PreMod}{PreMod}
\DeclareMathOperator{\AdaLN}{AdaLN}

\DeclareMathOperator{\Warp}{Warp}
\newtheorem{proposition}{Proposition}
\newtheorem{corollary}{Corollary}
\usepackage[colorlinks=true,linkcolor=blue,citecolor=blue,urlcolor=blue,linktocpage=true]{hyperref}

\begin{document}

\title{Ego-Dynamics-Augmented World Model for
Autonomous Driving with Zero-Shot
Cross-Embodiment Adaptation}

\author{Zhidong Wang, Jingsong Liang, Zirui Li, \IEEEmembership{Graduate Student Member,~IEEE}, Zhan Chen, Han Yu, \IEEEmembership{Senior Member,~IEEE}, and Chen Lv, \IEEEmembership{Senior Member,~IEEE}
\thanks{This work was supported in part by Nanyang Technological University. (\textit{Corresponding author: Chen Lv.})}
\thanks{Zhidong Wang, Jingsong Liang, Zirui Li, Zhan Chen and Chen Lv are with the School of Mechanical and Aerospace Engineering, Nanyang Technological University, Singapore
(e-mail: \texttt{zhidong001@e.ntu.edu.sg}; \texttt{jingsong002@e.ntu.edu.sg}; \texttt{zirui.li@ntu.edu.sg}; \texttt{zhan014@e.ntu.edu.sg}; \texttt{lyuchen@ntu.edu.sg}).

Zhidong Wang is also with the Collaborative Initiative, Interdisciplinary Graduate Programme, Nanyang Technological 
University, Singapore.

Han Yu is with the College of Computing and Data Science, Nanyang Technological University, Singapore (e-mail: \texttt{han.yu@ntu.edu.sg}).}}

\markboth{}{}


\maketitle

\begin{abstract}
End-to-end autonomous driving requires generalization ability across platforms with dissimilar physical characteristics. The chassis defines the physical embodiment of each platform, and real-world fleets span sub-tonne microcars to bus-class vehicles. Consequently, the driving stack must either be retrained per platform or adapt to the underlying chassis dynamics online.
World model (WM)-based reinforcement learning offers a sample-efficient path toward end-to-end autonomous driving on egocentric bird's-eye-view (BEV) representations, but its effectiveness hinges on how faithfully the WM captures the ego vehicle's dynamics.
This work identifies a structural bottleneck in BEV-based WMs: observation transitions entangle ego-motion with scene dynamics, consuming modeling capacity at the cost of imagination accuracy. This burden is embodiment-dependent: dissimilar chassis produce different observation warps under the same control input.
The proposed \textbf{DynaDreamer} addresses this bottleneck by conditioning the WM's latent distributions on a physics-informed ego-dynamics context derived from a lateral dynamics model with a neural tire force formulation. This context is extracted online via a neural-ODE encoder-decoder that simultaneously identifies the underlying chassis parameters. Information-theoretic analysis confirms that this conditioning removes the ego-motion terms from both the WM's transition entropy and its prior--posterior KL divergence.
The identified physical parameterization enables zero-shot cross-embodiment adaptation across a dynamically diverse fleet without per-platform retraining.
Simulation results show 28\% and 43\% improvements in driving task success rates over the strongest baseline in urban and highway scenarios, and the advantage over the base Transformer WM reaches up to 73\% when extrapolating to unseen chassis.
\end{abstract}


\section{Introduction}

\IEEEPARstart{E}{nd}-to-end autonomous driving maps raw sensor inputs directly to planning outputs such as waypoints or low-level control commands~\cite{ADSurvey}, increasingly operating on bird's-eye-view (BEV) representations fused from surround-view cameras as a unified metric canvas for joint perception--planning optimization~\cite{BEVFormer}. Deploying such a stack across a production fleet raises a cross-embodiment challenge that is primarily dynamical. Although sensing and actuation interfaces are shared across platforms, the underlying chassis range from sub-tonne microcars to bus-class vehicles and exhibit vastly different dynamic characteristics. The driving policy must therefore remain valid across this entire spectrum. Learning such a dynamics-aware policy is data-intensive: imitation learning suffers from causal confusion and limited out-of-distribution generalization~\cite{VAD}, while reinforcement learning (RL) restores closed-loop competence at the cost of sample efficiency on high-dimensional BEV observations~\cite{BEVRL}. World models (WMs) alleviate this bottleneck by learning compact latent dynamics in which imagined trajectories can be rolled out without real environment interaction, driving rapid progress in WM-based RL for autonomous driving~\cite{9934803,11250785}. However, such imagined rollouts are useful only when they faithfully reflect the ego vehicle's physical dynamics; any mismatch between imagined and actual vehicle behavior propagates directly into suboptimal or unsafe control.
\IEEEpubidadjcol

In the BEV setting, the ego vehicle is fixed at the image center; the change between consecutive frames therefore superimposes the \emph{ego-induced flow} of the viewpoint shift onto the \emph{object-induced flow} of surrounding agents. The WM must spend substantial capacity recovering the former at the expense of modeling the latter. Worse, this recovery reverses causality: ego-motion is physically determined by the executed action and the chassis dynamics, but a pixel-only WM can only infer it a posteriori from observations. This burden is moreover chassis-dependent: vehicles with dissimilar dynamics executing the same action produce different observation warps, a challenge recognized even in platoon control~\cite{10091195,10012062}. A WM that identifies chassis dynamics online would therefore enable zero-shot cross-embodiment adaptation and cover the entire fleet with a single trained model.

This work augments the WM with explicit ego-dynamics conditioning, freeing its modeling capacity for the scene dynamics beyond the ego vehicle's motion. The developed method, \textbf{DynaDreamer} (\textbf{dyna}mics-augmented \textbf{Dreamer}), extracts a compact ego-dynamics context from the ego-state history via a physics-informed neural ordinary differential equation (ODE) encoder-decoder and injects it into the WM's prior and posterior latent distributions. This conditioning is grounded in identifiable chassis parameters rather than a generic learned context, enabling zero-shot cross-embodiment adaptation across dynamically diverse fleets with a single trained model. The contributions are as follows:
\begin{itemize}
    \item \textbf{Theoretical foundation.} An information-theoretic analysis formalizes the structural bottleneck in BEV-based WM learning and proves that ego-dynamics conditioning reduces both the transition entropy and the achievable prior--posterior Kullback--Leibler (KL) divergence by their respective ego-motion terms.
    \item \textbf{Ego-dynamics-augmented world model.} The WM's prior and posterior distributions are conditioned on the ego-dynamics context via adaptive modulation, with this conditioning kept synchronized with the evolving ego state throughout imagination, tightening the prior--posterior gap under diverse chassis dynamics.
    \item \textbf{Zero-shot cross-embodiment adaptation.} The ego-dynamics context is grounded in identifiable chassis parameters, enabling online re-identification of unseen vehicles for dynamically diverse fleets and physically consistent imagined rollouts without per-platform retraining.
\end{itemize}
 
The remainder of this paper is organized as follows. Section~\ref{sec:related} reviews the related work. Section~\ref{sec:prelim} provides the preliminaries on WM-based RL. Section~\ref{sec:method} elaborates the methodology of DynaDreamer. Section~\ref{sec:theory} presents theoretical insights into the ego-dynamics prior. Section~\ref{sec:experiments} presents experiments, including the simulation setup, comparative results and analysis. Section~\ref{sec:conclusion} concludes the paper.

\section{Related Work}
\label{sec:related}
\subsection{World Models for Autonomous Driving}
\label{subsec:rw-wm-ad}
WMs have been adopted in end-to-end autonomous driving under two main formulations. The first learns a world-dynamics model and then trains a policy via IL on either low-level commands~\cite{drivedreamer} or ego-waypoint trajectories~\cite{MultiviewVisualForecastingandPlanning}. The second performs RL directly in the WM's latent space, covering offline RL on large-scale driving logs~\cite{DreamerAD} and online RL within high-fidelity simulators~\cite{think2drive}, such as Dreamer~\cite{Dreamer2025}. Latent sequential models have likewise been exploited for urban driving with multi-modal fusion~\cite{9934803}, and refining decisions inside a learned WM further improves end-to-end driving~\cite{11250785}. \citet{rawdrive} couples BEV perception features with a teacher-student scheme based on Dreamer to learn end-to-end control, yet still suffers from instability in generating low-level commands.
 
Despite this progress, several limitations remain in WMs. The latent could encode substantial task-irrelevant information that dilutes downstream signals, where masked semantic representations help suppress~\cite{SemanticMaskedWorldModel}. Compounding prediction error over long-horizon rollouts has been mitigated by decoupling spatiotemporal factors in autoregressive diffusion WMs~\cite{EponaDiffusionWM}. Yet all these advances refine a world model learned purely from pixel reconstruction, with no physical inductive bias~\cite{PhysicsConsistency}. The ego vehicle's own motion therefore remains entangled with the environment dynamics in the shared egocentric latent, a gap that DynaDreamer addresses.
 
\subsection{Augmenting World Models with Ego-Dynamics}
\label{subsec:rw-physics}
The egocentric BEV latent entangles ego-motion with environment dynamics. A prominent line of work addresses this entanglement by modeling the ego dynamics separately from the surroundings. Representative strategies include pairing a deterministic kinematic ego model with a stochastic environment model~\cite{SeparatingEgoWorld}, learning a separate predictive WM for ego trajectories~\cite{PredictiveIndividualWorldModel}, decoupling ego kinematics in latent space~\cite{VehicleDynamicsEmbeddedWM}, applying action-conditioned decoding~\cite{DriveX}, and disentangling controllable from observable factors~\cite{DisentangledAgentEnv}. However, none of these methods feeds the ego-dynamics signal back to condition the WM's prior and posterior distributions, and none ties the signal to identifiable physical parameters. In parallel, vehicle-dynamics research identifies chassis parameters within hybrid physics--data models~\cite{10566037} and embeds ODE-based physical knowledge into neural state estimators~\cite{11568729}, yet these identification techniques remain unconnected to world-model learning.
 
Dynamically dissimilar vehicles moreover require different control strategies despite similar observations. The resulting chassis-parameter variability has motivated model-free controllers that abandon the dynamics model altogether~\cite{LiuRLPathFollowing2024}. Context-based adaptation methods~\cite{MetaRL,DistributionDrift} incorporate generic latents but lack physical parameterization. The closest prior work applies per-vehicle physical conditioning at the policy level alone, without a predictive WM~\cite{OneModeltoDriftThemAll}. DynaDreamer closes this gap by jointly identifying chassis parameters, injecting them into the WM's prior and posterior, and maintaining this conditioning throughout imagination, unifying ego-dynamics identification and world-model learning.
\section{Preliminaries}
\label{sec:prelim}
 
Consider the task of autonomous driving formulated as a partially observable Markov decision process (POMDP). At each step $t$, the agent receives an egocentric BEV observation $\vobs_t$ with the ego vehicle fixed at the image center and executes a continuous control command $\vact_t=[a_{\mathrm{acc}},a_{\mathrm{steer}}]^\top$, where $a_{\mathrm{acc}}$ is a longitudinal command and $a_{\mathrm{steer}}$ a lateral command. The environment returns a scalar reward $r_t$ and a continuation flag $n_t\in\{0,1\}$, and the objective is to maximize the expected discounted return $\mathbb{E}\left[\sum_t \gamma^t r_t\right]$ with discount factor $\gamma\in(0,1)$.
 
To handle the high-dimensional pixel space and unknown environment dynamics, a WM learns a compact state $\mathbf{s}_t\triangleq(\vhid_t,\vlat_t)$ as a differentiable proxy of the environment, where $\vhid_t$ is a deterministic state summarizing history and $\vlat_t$ is a stochastic latent~\cite{Dreamer2025}. The model comprises five components: a posterior $\vlat_t\sim q_\phi(\vlat_t\mid\vhid_t,\vobs_t)$, a prior $\hat{\vlat}_t\sim p_\phi(\hat{\vlat}_t\mid\vhid_t)$, a reward predictor $\hat r_t\sim p_\phi(\hat r_t\mid\vhid_t,\vlat_t)$, a continuation predictor $\hat n_t\sim p_\phi(\hat n_t\mid\vhid_t,\vlat_t)$, and a decoder $\hat\vobs_t\sim p_\phi(\hat\vobs_t\mid\vhid_t,\vlat_t)$. The posterior infers $\vlat_t$ with access to the current observation, whereas the prior predicts it from history alone; their mismatch is penalized by two KL terms, giving the WM loss:
\begin{equation}
\begin{aligned}
\mathcal{L}_{\mathrm{WM}} &= \mathcal{L}_{\mathrm{rec}} + \mathcal{L}_r + \mathcal{L}_n + \beta_{\mathrm{dyn}}\mathcal{L}_{\mathrm{dyn}} + \beta_{\mathrm{rep}}\mathcal{L}_{\mathrm{rep}},\\[4pt]
\mathcal{L}_{\mathrm{dyn}} &= \max\left(1,\KL\left[\sg\left(q_\phi\right)\,\|\,p_\phi\right]\right),\\
\mathcal{L}_{\mathrm{rep}} &= \max\left(1,\KL\left[q_\phi\,\|\,\sg\left(p_\phi\right)\right]\right),
\end{aligned}
\label{eq:wm-loss}
\end{equation}
where $\mathcal{L}_{\mathrm{rec}},\mathcal{L}_r,\mathcal{L}_n$ are the reconstruction, reward, and continuation losses, $\sg(\cdot)$ is the stop-gradient operator, and $\beta_{\mathrm{dyn}},\beta_{\mathrm{rep}}>0$ are KL-balancing coefficients. The deterministic state is advanced by a sequence model over past latents and actions, $\vhid_{1:T}=f_\phi\!\left(\Embed\left(\vlat_{1:T},\vact_{1:T}\right)\right)$, instantiated here as a causal Transformer with a key--value (KV) cache over multi-hot categorical latents~\cite{STORM2023}. This architecture enables parallel training and long-horizon credit assignment and serves as the backbone that DynaDreamer augments with ego-dynamics.
 
The driving policy is learned entirely within imagined rollouts. Starting from a state $\mathbf{s}_t$, the actor $\pi_\theta(\vact_t\mid\mathbf{s}_t)$ samples an action $\vact_t$; the sequence model then advances $\vhid_{t+1}$; and the prior, reward, and continuation heads produce $\hat{\vlat}_{t+1}$, $\hat r_{t+1}$, and $\hat n_{t+1}$, respectively. The actor and a critic $V_\rho(\mathbf{s}_t)$ are optimized on these rollouts to maximize the $\lambda$-return:
\begin{equation}
G^\lambda_t = \hat r_t + \gamma\hat n_t\left[\left(1-\lambda\right)V_\rho\left(\mathbf{s}_{t+1}\right) + \lambda G^\lambda_{t+1}\right],
\label{eq:lambda-return}
\end{equation}
where $\lambda\in[0,1]$ is the trace-decay coefficient. A key consequence of this design is that imagination only rolls out the prior, and any dynamics that the prior cannot anticipate are irrecoverable within a rollout, which is a central property for the analysis in Section~\ref{sec:theory}.

\begin{figure*}[!t]
\centering
\includegraphics[width=\textwidth]{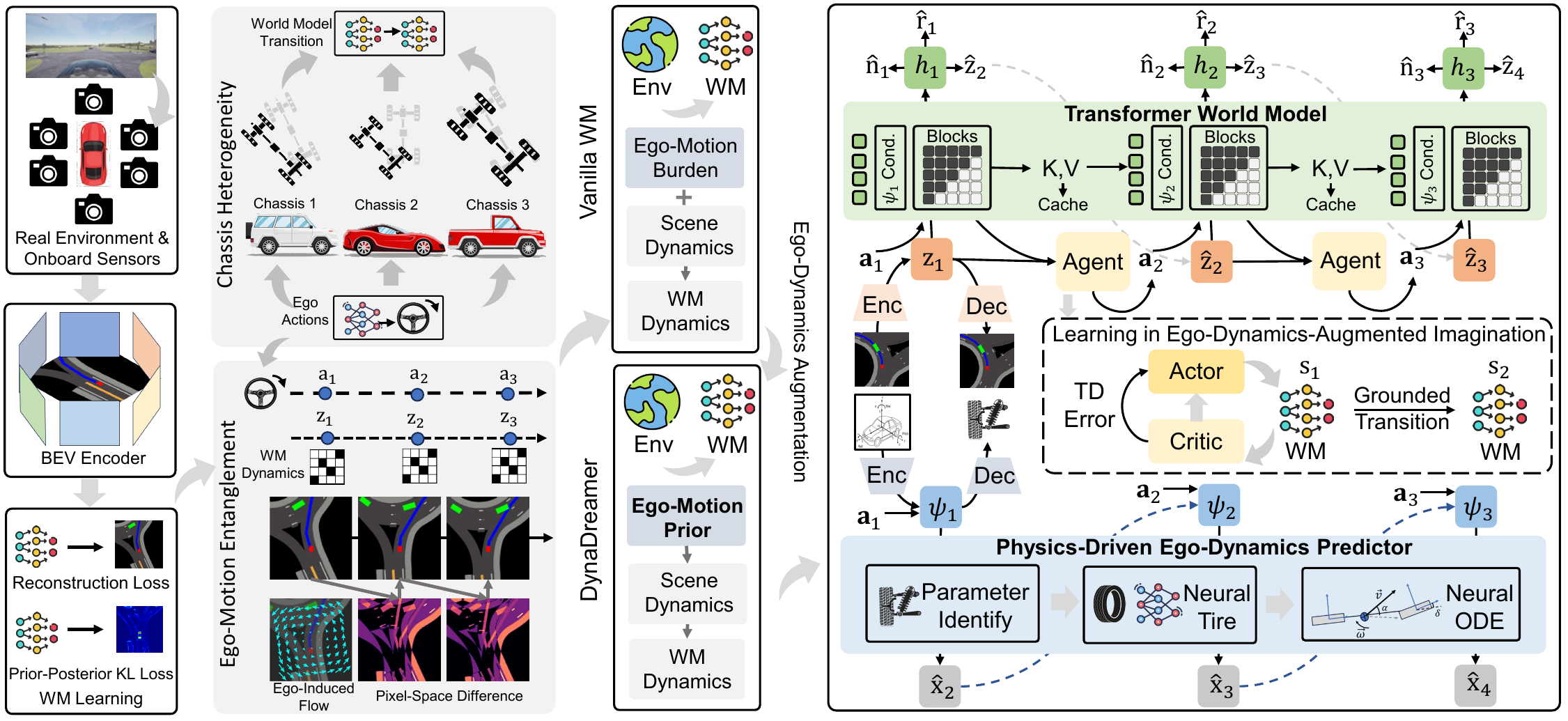}
\caption{Overview of DynaDreamer. In egocentric BEV WMs, ego-motion induces a dominant, chassis-dependent flow field that entangles with the scene dynamics, inflating both the reconstruction loss and the prior--posterior KL divergence and preventing transfer across dynamically dissimilar vehicles. DynaDreamer replaces this implicit burden with an explicit ego-motion prior grounded in identifiable chassis physics: a physics-informed neural-ODE encoder-decoder extracts the ego-dynamics context $\vdctx_t$, which conditions the Transformer WM via AdaLN modulation and is propagated through imagination via a modulated KV cache.}
\label{fig:overall-framework}
\end{figure*}

\section{Methodology}
\label{sec:method}
 
Resolving the ego-motion burden with an ego-dynamics conditioning signal imposes three requirements that jointly determine the architecture. First, \emph{identifiability}: the conditioning must be recoverable online from onboard state--action histories and grounded in physical chassis parameters to enable transfer to unseen vehicles. Second, \emph{distributional injection}: the conditioning must enter the WM's prior and posterior latent distributions, because routing the same information only to the policy leaves the prior--posterior bottleneck unresolved (Section~\ref{sec:theory}). Third, \emph{imagination-time synchronization}: the conditioning must track the evolving ego state during rollouts, where no observations are available. Sections~\ref{subsec:vd-enc}--\ref{subsec:vd-imag} address these three requirements in turn: a physics-informed neural-ODE encoder-decoder produces the ego-dynamics context (VD-context) $\vdctx_t$ and identified chassis parameters $\hat\vparams_t$; Adaptive Layer Normalization (AdaLN) modulation with Rotary Position Embedding (RoPE) conditions both latent distributions on $\vdctx_t$; and an ODE predictor with a modulated KV cache updates $\vdctx_t$ throughout imagination. Fig.~\ref{fig:overall-framework} provides an overview.

\subsection{Ego-Dynamics Encoder-Decoder}
\label{subsec:vd-enc}
 
\subsubsection{Overview}
An ego-dynamics encoder-decoder is designed to extract and update a compact dynamics context $\vdctx_t$ from a sliding window of ego-vehicle states and actions. Let $\vstate_t=[v_x,v_y,a_x,a_y,\omega,\dot\omega]^\top$ be the ego-vehicle states (longitudinal/lateral velocities, accelerations, yaw rate, and yaw acceleration), and $\vact_t=[a_{\mathrm{acc}},a_{\mathrm{steer}}]^\top$ be the corresponding normalized control inputs. Define the $K$-step context window $\mathcal{W}_t=(\vstate_{t-K+1:t},\vact_{t-K+1:t})$. The encoder-decoder comprises five sub-modules:
\begingroup
\allowdisplaybreaks
\begin{subequations}
\label{eq:vd-modules}
\begin{alignat}{2}
\text{Physics Aug.:}\quad
& \vphys_t
&= {}& \phi_{\mathrm{phys}}(\mathcal{W}_t;\vparams),
\label{eq:vd-phys-aug}\\
\text{VD-context Enc.:}\quad
& (\vdctx_t,\hat\vparams_t)
&= {}& \Enc_\xi(\mathcal{W}_t,\vphys_t),
\label{eq:vd-context-encoder}\\
\text{Neural Tire:}\quad
& F_y
&= {}& \TireF(\alpha,v_x;C),
\label{eq:vd-tire-model}\\
\text{Bicycle ODE:}\quad
& \hat\vstate_{t+1}^{3}
&= {}& \ODE(\vstate_t^{3},\vact_t;\hat\vparams_t),
\label{eq:vd-bicycle-ode}\\
\text{Dynamics Pred.:}\quad
& \hat\vstate_{t+1}
&= {}& \dynpred(\vdctx_t,\vstate_t,\vact_t;\hat\vparams_t),
\label{eq:vd-dyn-predictor}
\end{alignat}
\end{subequations}
\endgroup
where $\vparams$ is the chassis parameter vector, $\vphys_t$ is the physics-augmented feature vector at time $t$, $\alpha$ is the tire slip angle, $C$ is the cornering stiffness coefficient, and $\vstate_t^3=[v_x,v_y,\omega]^\top$ is the three-degree-of-freedom (3-DOF) lateral dynamics state. Fig.~\ref{fig:method-arch}(a) illustrates the complete encoder-decoder pipeline.
 
\begin{algorithm2e}[!t]
\small
\DontPrintSemicolon
\SetAlgoLined
\KwIn{Normalized batch $(\vstate_{1:T},\vact_{1:T})$; window size $K$; ground-truth chassis $\vparams^{\star}$; collision mask $b_{t+k}\in\{0,1\}$.}
\KwOut{$\vdctx_{1:T}$,\ $\hat\vparams_{1:T}$,\ $\mathcal{L}_{\mathrm{aux}}$,\ $\mathcal{L}_{\mathrm{param}}$.}
Build causal windows $\mathcal{W}_t=(\vstate_{t-K+1:t},\vact_{t-K+1:t})$ for $t=1,\ldots,T$.\;
Compute initial physics features $\vphys^{(1)}$ using $\vparams_{\mathrm{nom}}$.\;
$(\vdctx_{1:T}^{(1)},\hat\vparams_{1:T}^{(1)})\leftarrow\Enc_\xi(\mathcal{W}_{1:T},\vphys^{(1)})$.\;
Refine $\vphys^{(2)}$ using estimated $\hat\vparams_{1:T}^{(1)}$.\;
$(\vdctx_{1:T},\hat\vparams_{1:T})\leftarrow\Enc_\xi(\mathcal{W}_{1:T},\vphys^{(2)})$.\;
Initialize $\mathcal{L}_{\mathrm{aux}}\leftarrow 0$.\;
\For{each rollout origin $t$}{
    Initialize rollout from ground-truth: $\hat{\vstate}_t\leftarrow\vstate_t$.\;
    \For{$k=1$ \KwTo $K_{\mathrm{pred}}$}{
        $\hat{\vstate}_{t+k}\leftarrow\dynpred(\vdctx_t,\hat{\vstate}_{t+k-1},\vact_{t+k-1};\hat\vparams_t)$.\;
        Accumulate masked prediction error: $\mathcal{L}_{\mathrm{aux}}\mathrel{+}=b_{t+k}\|\hat{\vstate}_{t+k}-\vstate_{t+k}\|_1$.\;
        Detach $\hat{\vstate}_{t+k}$ from computation graph.\
    }
}
Calculate $\mathcal{L}_{\mathrm{aux}}$ and $\mathcal{L}_{\mathrm{param}}$.\;
\Return $\vdctx_{1:T},\ \hat\vparams_{1:T},\ \mathcal{L}_{\mathrm{aux}},\ \mathcal{L}_{\mathrm{param}}$.\;
\caption{Iterative ego-dynamics context encoding and multi-step prediction.}
\label{alg:vd-encdec}
\end{algorithm2e}

\subsubsection{Ego-dynamics encoder and parameter head}
To augment inputs with physics information, six physics-informed features are derived from $\mathcal{W}_t$:
\begin{equation}
\phi_{\mathrm{phys}}\left(\mathcal{W}_t;\vparams\right) = \bigl[\beta,\;v,\;\kappa,\;\tilde F_x,\;\tilde F_y,\;\tilde M_z\bigr],
\label{eq:phys-feat}
\end{equation}
where $\beta=\arctan(v_y/v_x)$ is the side-slip angle, $v=\|(v_x,v_y)\|$ is the speed, $\kappa=\omega/v$ is the trajectory curvature, and $(\tilde F_x,\tilde F_y,\tilde M_z)$ are coarse vehicle-level force and moment estimates derived from the identified mass and yaw inertia. A Gated Recurrent Unit (GRU) encodes the augmented sequence $[\vstate_{t-K+1:t};\vphys_t;\vact_{t-K+1:t}]$ into $\vdctx_t$.
 
The parameter vector $\vparams=[m,\,I_z,\,l_f,\,l_r,\,C_f,\,C_r,\,\delta_{\max}]^\top$ collects the chassis quantities governing lateral vehicle dynamics: mass, yaw inertia, axle distances, cornering stiffnesses, and maximum steering angle. Computing $\vphys_t$ requires $\vparams$, which is itself the identification target; encoding therefore proceeds in two weight-shared stages that bootstrap the physics features from nominal parameters $\vparams_{\mathrm{nom}}$ and refine them with the first-stage estimate to yield the final $\vdctx_t$ and $\hat\vparams_t$. The parameter head maps $\vdctx_t$ to $\hat\vparams_t$ through a $\tanh$-bounded log-space multilayer perceptron (MLP), constraining $\hat\vparams_t$ to a physically plausible multiplicative offset from $\vparams_{\mathrm{nom}}$.
 
\subsubsection{Neural tire model}
A tire model maps axle slip angles and cornering stiffnesses to lateral forces, forming the physical link between $\hat\vparams_t$ and the bicycle ODE.
A linear tire model cannot capture the nonlinear saturation that dominates lateral force at high slip angles, yet a purely data-driven replacement discards the causal structure required for stable ODE integration.
To reconcile fidelity with physical tractability, the lateral force is modeled as a linear prior plus a zero-initialized neural correction:
\begin{equation}
\TireF\left(\alpha,v_x;C\right) \;=\; -C\alpha \;+\; \Delta_{\mathrm{nn}}\left(\alpha,v_x,C\right),
\label{eq:neural-tire}
\end{equation}
where $\Delta_{\mathrm{nn}}$ is a two-hidden-layer MLP with $\tanh$ activations. Zero initialization ensures that at the start of training $\Delta_{\mathrm{nn}}\equiv 0$, reducing the dynamics to a linear bicycle model; the network then progressively captures saturation and speed-dependent stiffness. The per-axle stiffnesses $C\in\{C_f,C_r\}$ and axle offsets $(l_f,l_r)$ are drawn from $\hat\vparams_t$; the front and rear slip angles derive from $\vstate_t^3$:
\begin{equation}
\alpha_f = \delta - \arctan\tfrac{v_y+l_f \omega}{v_x}, \qquad
\alpha_r = -\arctan\tfrac{v_y-l_r \omega}{v_x},
\label{eq:slip-angles}
\end{equation}
where $\delta=a_{\mathrm{steer}}\cdot g(v_x)\cdot\delta_{\max}$ is the estimated front-wheel steering angle, and $g(v_x)$ is a fixed piecewise-linear speed-attenuation curve that accounts for the speed-dependent steering gain. Although the true speed-dependent gain varies across chassis, sharing a fixed $g(v_x)$ remains valid: the identified $\delta_{\max}$ rescales the curve per vehicle, and the remaining chassis-specific mismatch is absorbed by the tire correction $\Delta_{\mathrm{nn}}$ and the residual predictor of Eq.~\eqref{eq:dynpred}.
 
\subsubsection{Bicycle model ODE}
A single-track bicycle ODE combines the lateral forces from the tire model with the identified parameters $\hat\vparams_t$ to forward-predict the kinematic states.
Let $F_{yf}=\TireF(\alpha_f,v_x;C_f)$ and $F_{yr}=\TireF(\alpha_r,v_x;C_r)$ denote the front and rear axle lateral forces given by Eq.~\eqref{eq:neural-tire}. Let $F_{xr}=a_{\mathrm{acc}}\,m\,a_{\max}$ denote the rear longitudinal force, where $a_{\max}$ is a fixed acceleration limit, and let $\delta$ denote the front wheel angle defined below Eq.~\eqref{eq:slip-angles}.
The evolution of $\vstate_t^3$ is then governed by the single-track bicycle model \cite{4162483}:
\begin{equation}
\begin{aligned}
\dot v_x &= \frac{F_{xr}-F_{yf}\sin\delta}{m}+v_y \omega,\\
\dot v_y &= \frac{F_{yf}\cos\delta+F_{yr}}{m}-v_x \omega,\\
\dot \omega   &= \frac{l_f F_{yf}\cos\delta - l_r F_{yr}}{I_z}.
\label{eq:bicycle-ode}
\end{aligned}
\end{equation}
 
The decoded parameter vector $\hat\vparams_t$ supplies all vehicle-specific quantities ($m, I_z, l_f, l_r$). Eq.~\eqref{eq:bicycle-ode} is integrated with a forward Euler step of size $\Delta t$; the remaining three state components $(a_x,a_y,\dot\omega)$ are recovered by finite differences.
 
The physics-informed ODE embeds causal structure to prevent physically implausible rollouts and improve data efficiency. Its explicit parameter interface further allows $\hat\vparams_t$ to be supervised by ground-truth chassis parameters, tying $\vdctx_t$ to identifiable physical quantities.
 
\subsubsection{Dynamics predictor}
The single-track bicycle ODE provides a physically grounded prior but cannot represent higher-order effects excluded by the single-track assumption. A neural residual conditioned on $\vdctx_t$ augments the physics prior to capture these unmodeled dynamics.
$\dynpred(\vdctx_t,\vstate_t,\vact_t;\hat\vparams_t)$ composes the bicycle ODE step with a zero-initialized neural residual operating in normalized state space:
\begin{equation}
\hat{\vstate}_{t+1} \;=\; \ODE\left(\vstate_t^{3},\vact_t;\hat\vparams_t\right) \;+\; \eta_\xi\left(\vdctx_t,\vstate_t,\vact_t\right),
\label{eq:dynpred}
\end{equation}
where $\ODE(\cdot)$ denotes the Euler-discretized bicycle step, and $\eta_\xi$ is a residual MLP block initialized to zero so that training starts from the pure physics prediction. The neural residual progressively absorbs suspension effects, load transfer, and other modeling errors that the single-track approximation cannot represent.
 
\subsubsection{Auxiliary losses}
The ego-dynamics neural-ODE encoder-decoder is trained with two independently weighted losses, multi-step prediction loss $\mathcal{L}_{\mathrm{aux}}$ and parameter identification loss $\mathcal{L}_{\mathrm{param}}$; stop-gradient on $\vdctx_t$ isolates their gradients within this module.
$\mathcal{L}_{\mathrm{aux}}$ is a uniformly weighted $K_{\mathrm{pred}}$-step rollout loss. Starting from the ground-truth state $\vstate_t$, predictions unroll autoregressively with gradients severed between steps through the ODE chain, which prevents unstable gradient amplification through repeated ODE integration while still exposing the encoder to compounding-error feedback:
\begin{equation}
\mathcal{L}_{\mathrm{aux}} \;=\; \tfrac{1}{K_{\mathrm{pred}}}\sum_{k=1}^{K_{\mathrm{pred}}} b_{t+k}\,\|\hat{\vstate}_{t+k}-\vstate_{t+k}\|_1,
\label{eq:aux-loss}
\end{equation}
where $\hat{\vstate}_{t+k}$ is the $k$-step prediction and $b_{t+k}\in\{0,1\}$ is a binary mask that zeroes out steps at or after a collision boundary. $\mathcal{L}_{\mathrm{param}}$ is defined in log space:
\begin{equation}
\mathcal{L}_{\mathrm{param}} \;=\; \left\|\,\log\tfrac{\hat\vparams_t}{\vparams_{\mathrm{nom}}} - \log\tfrac{\vparams^{\star}}{\vparams_{\mathrm{nom}}}\,\right\|_2^2,
\label{eq:param-loss}
\end{equation}
where $\vparams^{\star}$ is the ground-truth chassis parameter vector. $\mathcal{L}_{\mathrm{aux}}$ and $\mathcal{L}_{\mathrm{param}}$ force $\vdctx_t$ to carry actionable, physically grounded information. Algorithm~\ref{alg:vd-encdec} summarizes the complete pipeline of the ego-dynamics encoder-decoder, which delivers $\vdctx_t$ as a self-consistent, physics-grounded ego-dynamics prior ready to be injected into the WM.
 
\begin{figure}[!t]
\centering
\includegraphics[width=\columnwidth]{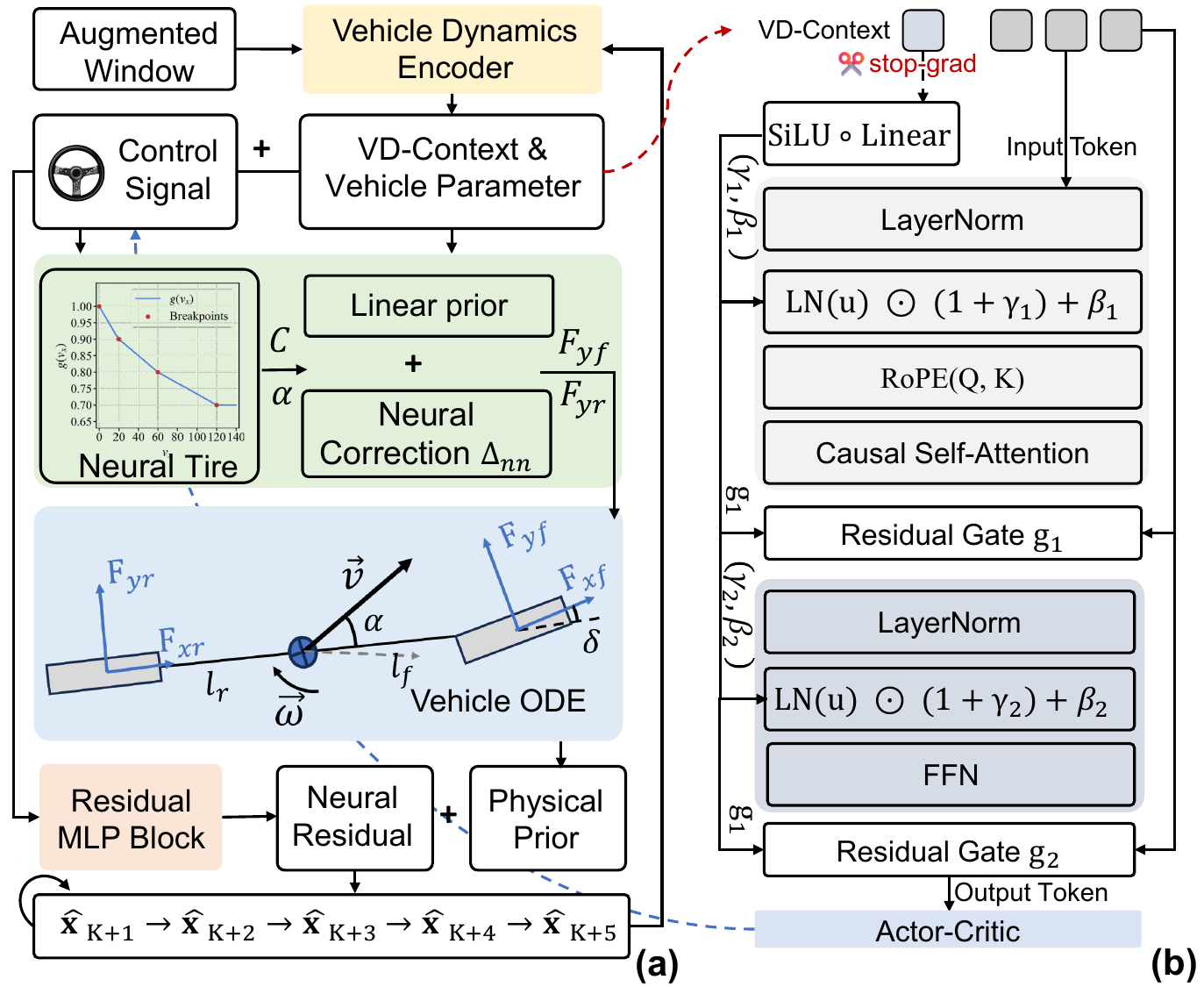}
\caption{Architectures of (a) the physics-informed Neural-ODE encoder-decoder and (b) the VD-context-modulated Transformer block. (a) An augmented $K$-step window of ego states and actions is encoded into the ego-dynamics context (VD-context) $\vdctx_t$; the dynamics predictor is driven by a neural ODE to roll out future ego states. (b) Each Transformer sub-block is wrapped with an AdaLN layer whose scale, shift, and gate are produced from a stop-gradient $\vdctx_t$, while RoPE keeps position encoding orthogonal to the modulation channel.}
\label{fig:method-arch}
\end{figure}

\subsection{Ego-Dynamics-Modulated World Model}
\label{subsec:vd-wm}
 
To reshape both the prior and posterior latents at every step, $\vdctx_t$ is integrated into the Transformer WM. The Transformer WM of Section~\ref{sec:prelim} uses homogeneous pre-norm blocks and absolute position embeddings, neither of which can adapt to the ego-dynamics heterogeneity across vehicles. Each Transformer block is therefore restructured to accept $\vdctx_t$ through AdaLN modulation, while RoPE decouples position information from the modulation channel, as depicted in Fig.~\ref{fig:method-arch}(b).
 
\subsubsection{AdaLN modulation}
Following the Diffusion Transformer (DiT)~\cite{Peebles_2023_ICCV}, every attention sub-block and feed-forward network (FFN) sub-block in the sequence model is replaced by an AdaLN layer whose affine parameters and residual gate are produced from $\vdctx_t$:
\begin{equation}
\begin{aligned}
\tilde{\mathbf{u}} &= \LN\left(\mathbf{u}\right)\odot\left(1+\boldsymbol{\gamma}\left(\vdctx_t\right)\right)+\mathbf{b}\left(\vdctx_t\right),\\
\mathbf{u} &\leftarrow \mathbf{u} + \mathbf{g}\left(\vdctx_t\right)\odot\Block\left(\tilde{\mathbf{u}}\right),
\end{aligned}
\label{eq:adaln}
\end{equation}
where $\mathbf{u}$ is the input activation of the sub-block, $\LN(\cdot)$ denotes Layer Normalization, $\Block(\cdot)$ is the attention or FFN operation, and $(\boldsymbol{\gamma},\mathbf{b},\mathbf{g})$ are the scale, shift, and gate projections for the sub-block. Modulation weights are zero-initialized so that the block collapses to an affine identity before training, while gate biases are offset by $0.5$ to keep the sub-block contribution non-zero from step zero. Modulation is chosen over input-stem concatenation deliberately: multiplicative scale, shift, and gate act on every sub-block, letting the low-dimensional $\vdctx_t$ steer the entire sequence model, whereas a concatenated context must compete with high-dimensional observation tokens and conditions the computation only at the input; Section~\ref{sec:experiments} confirms this gap. Any position encoding that carries ego-dynamics information would corrupt the modulation channel. Position encoding must therefore remain strictly orthogonal to $\vdctx_t$.
 
\subsubsection{Position embedding}
Absolute position embeddings are replaced with RoPE~\cite{su2024roformer}, a relative-position scheme that applies per-head rotations to query and key vectors as a function of token position alone. The critical design property is orthogonality: the rotation angle depends solely on token position and carries no information about $\vdctx_t$, so position encoding and AdaLN modulation operate on strictly disjoint channels. The KV cache therefore remains safely reusable as $\vdctx_t$ evolves across imagination steps.
 
\subsubsection{Training objective and detached context}
The WM loss inherits the composition of Eq.~\eqref{eq:wm-loss}; the full DynaDreamer training objective is:
\begin{equation}
\mathcal{L}^{\mathrm{DD}} = \mathcal{L}_{\mathrm{WM}}\left(\vdctx_t\leftarrow\sg\left(\vdctx_t\right)\right) + \alpha_{\mathrm{aux}}\mathcal{L}_{\mathrm{aux}} + \alpha_{\mathrm{param}}\mathcal{L}_{\mathrm{param}},
\label{eq:dd-loss}
\end{equation}
where $\sg(\cdot)$ blocks gradient flow into the VD-context encoder, and $\alpha_{\mathrm{aux}},\alpha_{\mathrm{param}}>0$ are scalar loss-weighting coefficients for the dynamics prediction and parameter-identification terms.
 
The VD-context $\vdctx_t$ is passed through a stop-gradient before entering the Transformer. This isolation prevents the WM from back-propagating reconstruction and KL gradients into the encoder $\xi$, which would collapse $\vdctx_t$ onto dynamics-irrelevant pixel statistics. The encoder is thus trained exclusively by $\mathcal{L}_{\mathrm{aux}}$ and $\mathcal{L}_{\mathrm{param}}$, preserving the physics grounding of Section~\ref{subsec:vd-enc}. $\vdctx_t$ is updated at every time step on a sliding window rather than held constant, which aligns WM learning and rollout.
 
\subsection{Ego-Dynamics-Aligned Imagination Rollout}
\label{subsec:vd-imag}
 
During training, $\vdctx_t$ is computed from the ground-truth ego-state window $\mathcal{W}_t$ at every step. Imagination introduces three problems absent from training. First, future states are produced by the model rather than observed, so $\vdctx_t$ must be propagated by the physics predictor. Second, KV cache entries written at prior steps carry no dynamics imprint unless explicitly modulated before storage, leading to inconsistency with the current-step conditioning. Third, trajectories terminate at different steps in batched imagination, and continuing to update $\vdctx_t$ for terminated samples risks corrupting surviving ones.
 
\subsubsection{Modulated KV cache for autoregressive rollout}
At every imagined step, the raw key $\mathbf{k}_t$ and value $\mathbf{v}_t$ produced by the Transformer are passed through pre-modulation $\PreMod(\cdot\,;\vdctx_t)$ to yield the modulated pair $(\tilde{\mathbf{k}}_t,\tilde{\mathbf{v}}_t)$, which is then written into the cache. Entries already in the cache remain untouched, so each token carries a permanent imprint of the ego-dynamics state under which it was produced while maintaining $\mathcal{O}(H)$ autoregressive complexity.
 
\subsubsection{Online VD-context update via ODE rollout}
At the start of a rollout, the true chassis parameters are unavailable, so $\vdctx_{K_{\mathrm{ctx}}}$ at the rollout-start index $K_{\mathrm{ctx}}$ is seeded with nominal parameters $\vparams_{\mathrm{nom}}$, where $K_{\mathrm{ctx}}$ denotes the warm-up context length distinct from the VD-context window size $K$.
From step $K_{\mathrm{ctx}}$ onward, the identified parameters $\hat\vparams_t$ produced by the encoder replace the nominal prior, and the ego state is propagated by the physics-informed predictor:
\begin{equation}
\begin{aligned}
\hat{\vstate}_{t+1} &= \dynpred\left(\vdctx_t,\hat\vstate_t,\vact_t;\hat\vparams_t\right),\\
\vdctx_{t+1} &= \Enc_\xi\left(\mathcal{W}_{t+1},\,\phi_{\mathrm{phys}}\left(\mathcal{W}_{t+1};\hat\vparams_t\right)\right),
\label{eq:online-ctx}
\end{aligned}
\end{equation}
where $\mathcal{W}_{t+1}=(\hat\vstate_{t-K+2:t+1},\,\vact_{t-K+2:t+1})$ slides forward by replacing real states with imagined ones while retaining the executed actions.
 
At deployment, a vehicle with different chassis dynamics produces a distinct ego-state window $\mathcal{W}$ that shifts $\vdctx_t$ through the encoder and thereby steers the AdaLN modulation in Eq.~\eqref{eq:adaln}. The Transformer weights $\phi$ require no retraining to adapt to the new dynamic characteristics.
 
\subsubsection{Per-sample termination freezing}
In batched imagination, samples terminate at different steps. A per-sample binary alive mask $\nu_t\in\{0,1\}$ tracks whether each trajectory is still active. When a sample terminates ($\nu_t=0$), its $\vdctx_t$ is held at the last active value rather than updated, whereas live samples ($\nu_t=1$) continue to receive the refreshed context from the ODE predictor. This selective freezing prevents post-termination states from corrupting the WM forward pass of surviving samples. The actor only takes the WM latent $(\vhid_t,\vlat_t)$, keeping the policy gradient path independent of the encoder $\xi$.

Algorithm~\ref{alg:dd-combined} summarizes training and imagination. Together, these three subsections close a loop between physics and pixels. The encoder-decoder extracts physics state into $\vdctx_t$ and $\hat\vparams_t$; the AdaLN-RoPE Transformer lets $\vdctx_t$ reshape the latent distribution while keeping encoder gradients isolated; and the modulated imagination procedure keeps $\vdctx_t$ synchronized with the ODE predictor during rollout. The actor is thereby optimized in a WM whose imagined responses are aligned with physics-driven ego-dynamics.
 
\begin{algorithm2e}[t]
\small
\DontPrintSemicolon
\SetAlgoLined
\KwIn{replay buffer $\mathcal{D}$;\ warm-up context $(\vobs_{1:K_{\mathrm{ctx}}},\vact_{1:K_{\mathrm{ctx}}},\vstate_{1:K_{\mathrm{ctx}}})$, imagination horizon $H$}
\BlankLine
\tcc{Training}
Sample a batch $(\vobs_{1:T},\vact_{1:T},r_{1:T},n_{1:T},\vstate_{1:T},\vparams^{\star})\sim\mathcal{D}$.\;
$(\vdctx_{1:T},\hat\vparams_{1:T},\mathcal{L}_{\mathrm{aux}},\mathcal{L}_{\mathrm{param}})\leftarrow\textbf{Algorithm~\ref{alg:vd-encdec}}$.\;
$\vdctx'_{1:T}\leftarrow \sg(\vdctx_{1:T})$.\tcp*{stop gradient}
Encode observations: $\mathbf{e}_{1:T}=\Embed(\vlat_{1:T},\vact_{1:T})$.\;
$\vhid_{1:T}\leftarrow f_\phi\bigl(\mathbf{e}_{1:T};\ \RoPE,\ \AdaLN(\vdctx'_{1:T})\bigr)$.\;
Decode reconstruction, reward, and continuation heads.\;
Compute $\mathcal{L}_{\mathrm{WM}}$ from Eq.~\eqref{eq:wm-loss}.\;
$\mathcal{L}^{\mathrm{DD}}\leftarrow \mathcal{L}_{\mathrm{WM}} + \alpha_{\mathrm{aux}}\mathcal{L}_{\mathrm{aux}}+\alpha_{\mathrm{param}}\mathcal{L}_{\mathrm{param}}$.\;
Update $\phi$ via $\nabla_\phi\mathcal{L}_{\mathrm{WM}}$ and $\xi$ via $\nabla_\xi\left(\mathcal{L}_{\mathrm{aux}}+\mathcal{L}_{\mathrm{param}}\right)$.\;
\BlankLine
\tcc{Imagination rollout}
Reset KV cache; encode warm-up context to obtain $\vhid_{K_{\mathrm{ctx}}},\vlat_{K_{\mathrm{ctx}}}$.\;
Initialize $\vdctx_{K_{\mathrm{ctx}}}\leftarrow\Enc_\xi(\mathcal{W}_{K_{\mathrm{ctx}}},\phi_{\mathrm{phys}}(\mathcal{W}_{K_{\mathrm{ctx}}};\vparams_{\mathrm{nom}}))$;\ \ $\hat\vstate_{K_{\mathrm{ctx}}}\leftarrow\vstate_{K_{\mathrm{ctx}}}$;\ \ $\nu_{K_{\mathrm{ctx}}}\leftarrow 1$.\;
\For{$t=K_{\mathrm{ctx}}$ \KwTo $K_{\mathrm{ctx}}+H-1$}{
  $\vact_t\sim\pi_\theta(\vact\mid \vhid_t,\vlat_t)$.\;
  $(\mathbf{k}_t,\mathbf{v}_t)\leftarrow$ Transformer step using $\AdaLN(\vdctx_t)$.\;
  Write $(\tilde{\mathbf{k}}_t,\tilde{\mathbf{v}}_t)=\PreMod(\mathbf{k}_t,\mathbf{v}_t;\vdctx_t)$ to cache.\;
  Sample $\vlat_{t+1}\sim p_\phi(\cdot\mid\vhid_{t+1})$, predict $\hat r_{t+1},\hat n_{t+1}$.\;
  $\hat{\vstate}_{t+1}\leftarrow\dynpred(\vdctx_t,\hat\vstate_t,\vact_t;\hat\vparams_t)$;\ \ slide window.\;
  $\vdctx_{t+1}\leftarrow\Enc_\xi(\mathcal{W}_{t+1},\phi_{\mathrm{phys}}(\mathcal{W}_{t+1};\hat\vparams_t))$.\;
  $\nu_{t+1}\leftarrow \nu_t\cdot(1-\hat n_{t+1})$;\ \ $\vdctx_{t+1}\leftarrow\nu_t\,\vdctx_{t+1}+(1-\nu_t)\vdctx_t$.\;
}
Form imagined trajectory $\tau=(\vhid,\vlat,\vact,\hat r,\hat n)_{K_{\mathrm{ctx}}:K_{\mathrm{ctx}}+H}$ and $\lambda$-return of Eq.~\eqref{eq:lambda-return}.\;
Update $\pi_\theta,V_\rho$ on $\tau$.\;
\caption{WM and ego-dynamics predictor joint training and dynamics-aligned imagination rollout.}
\label{alg:dd-combined}
\end{algorithm2e}

\section{Theoretical Insights into the Ego-Dynamics Prior}
\label{sec:theory}
 
This section provides an information-theoretic analysis of why augmenting the world model with an ego-dynamics prior eases learning. The analysis adopts the following assumptions: (i)~Lipschitz vehicle dynamics, (ii)~additive observation noise, (iii)~a $\beta$-mixing observation--action process~\cite{rio2017asymptotic}, (iv)~a persistently exciting policy in the classical system-identification sense~\cite{ljung1999system}, whose state--action excitation renders dynamically distinct chassis distinguishable, and (v)~conditional independence of the object-induced flow $\Delta_t$ (defined in Eq.~\eqref{eq:flow-decomp}) from the ego-dynamics context given the current observation and action. The chassis context $c$, instantiated by the physical parameter vector $\vparams$ of Section~\ref{subsec:vd-enc}, depends only on the vehicle. $H(\cdot)$ and $I(\cdot\,;\cdot)$ denote the (differential) entropy and mutual information, respectively~\cite{cover1999elements}. The results are stated as propositions with proofs that emphasize the key reasoning steps.
 
\subsection{Egocentric Observation Model}
Let $\xi_t\in SE(2)$ be the ego-vehicle pose, and let $g_t\triangleq\xi_t^{-1}\xi_{t+1}$ denote the one-step relative ego motion expressed in the current ego frame. The vehicle dynamics determine $g_t$ as $g_t=G(\vstate_t,\vact_t;c)$. The BEV is rendered around the ego vehicle, so the static background transforms rigidly with the viewpoint change. Consecutive observations therefore obey the following decomposition:
\begin{equation}
\vobs_{t+1}=\Warp_{g_t}(\vobs_t)\ \oplus\ \Delta_t,
\label{eq:flow-decomp}
\end{equation}
where $\Warp_{g_t}(\vobs_t)$ is the ego-induced flow, i.e.\ the rigid image warp of $\vobs_t$ under the ego motion $g_t$; $\oplus$ denotes pixel-wise composition; and $\Delta_t$ is the object-induced flow, collecting the genuine motion of surrounding agents together with content entering or leaving the field of view. This decomposition mirrors the classical separation of an observed flow field into camera-motion and object-motion components~\cite{683770}. The first term is fixed once $g_t$ is known; only the second reflects what the WM must genuinely predict.
 
\subsection{The Ego-Motion Modeling Burden}
\begin{proposition}[Ego-motion modeling burden]
\label{prop:decomp}
Under the egocentric model of Eq.~\eqref{eq:flow-decomp}, the conditional entropy of the next observation is decomposed as:
\begin{equation}
H(\vobs_{t+1}\mid\vobs_t,\vact_t)=I(g_t;\vobs_{t+1}\mid\vobs_t,\vact_t)+H(\vobs_{t+1}\mid\vobs_t,\vact_t,g_t),
\label{eq:entropy-decomp}
\end{equation}
where the first term is the ego-motion burden, i.e.\ the predictive uncertainty that arises solely from not knowing $g_t$. Conditioning on any statistic that determines $g_t$ removes at least this quantity from the predictive entropy, with equality for a statistic informationally equivalent to $g_t$.
\end{proposition}
\noindent\textit{Proof.} The decomposition is a direct application of the chain rule of conditional entropy. The warp $\Warp_{g_t}$ is deterministic given $g_t$, so the mutual information term is precisely the burden a pixel-only model spends on ego-motion. The entropy reduction under additional conditioning holds because conditioning cannot increase entropy~\cite{cover1999elements}. \hfill$\square$
 
\begin{proposition}[Geometric amplification and chassis multimodality]
\label{prop:amp}
The ego-motion burden is (i)~geometrically amplified, as the warp displaces content at scene radius $\rho$ by $\approx\rho\,|\Delta\theta_t|$ for an ego rotation $\Delta\theta_t$, and (ii)~chassis-dependent:
\begin{equation}
H(\vobs_{t+1}\!\mid\!\vobs_t,\vact_t)=H(\vobs_{t+1}\!\mid\!\vobs_t,\vact_t,c)+I(c;\vobs_{t+1}\!\mid\!\vobs_t,\vact_t),
\label{eq:chassis-decomp}
\end{equation}
where the equality is the conditional mutual information identity applied to $c$. The non-trivial claim is $I(c;\vobs_{t+1}\mid\vobs_t,\vact_t)>0$ whenever dynamically distinct chassis are excited.
\end{proposition}
\noindent\textit{Proof.} The equality follows from the chain rule of conditional entropy. The strictly positive mutual information holds under assumptions~(i) and~(iv): distinct chassis produce distinct ego-motions $g_t$ under the same action, which in turn induce distinct observation warps. A pixel-only WM must therefore identify $c$ implicitly to remain Markov in $(\vobs_t,\vact_t)$. \hfill$\square$
 
\subsection{The Prior--Posterior Bottleneck and Its Resolution}
The Transformer WM exposes the full history through its KV cache, and the bottleneck therefore lies not in model capacity but in the information asymmetry between prior and posterior. The posterior has access to $\vobs_{t+1}$ and can therefore recover the realized ego-motion $g_t$ via Eq.~\eqref{eq:flow-decomp}. The observation-free prior, by contrast, must predict $g_t$ from history alone, creating an irreducible misalignment whenever $g_t$ varies across chassis.
 
\begin{proposition}[Prior--posterior bottleneck]
\label{prop:bottleneck}
For the latent WM with an optimal observation-free prior, the minimal expected dynamics KL at step $t{+}1$ satisfies:
\begin{equation}
\begin{aligned}
\min_{p}\ \mathbb{E}&\big[\KL\!\left(q_\phi(\vlat_{t+1}\mid\vhid_{t+1},\vobs_{t+1})\,\|\,p(\vlat_{t+1}\mid\vhid_{t+1})\right)\big]\\
&=I(\vlat_{t+1};\vobs_{t+1}\mid\vhid_{t+1})\ \ge\ I(\vlat_{t+1};g_t\mid\vhid_{t+1}).
\end{aligned}
\label{eq:kl-bound}
\end{equation}
Under chassis randomization, $\vhid_{t+1}$ does not determine $g_t$, so $I(\vlat_{t+1};g_t\mid\vhid_{t+1})>0$ is an irreducible prior--posterior misalignment. Conditioning the prior on $\vdctx_t$ (the VD-context of Section~\ref{subsec:vd-enc}) lowers the achievable bound to $I(\vlat_{t+1};\vobs_{t+1}\mid\vhid_{t+1},\vdctx_t)$, removing exactly the ego-motion misalignment:
\begin{equation}
\begin{aligned}
\Delta_{\KL}&=I(\vlat_{t+1};\vobs_{t+1}\mid\vhid_{t+1})-I(\vlat_{t+1};\vobs_{t+1}\mid\vhid_{t+1},\vdctx_t)\\
&=I(\vlat_{t+1};g_t\mid\vhid_{t+1})>0.
\end{aligned}
\label{eq:delta-kl}
\end{equation}
\end{proposition}
\noindent\textit{Proof.} The optimal observation-free prior is the aggregated posterior $p^\star(\vlat_{t+1}\mid\vhid_{t+1})=\mathbb{E}_{\vobs_{t+1}}[q_\phi(\vlat_{t+1}\mid\vhid_{t+1},\vobs_{t+1})]$, for which the average KL equals $I(\vlat_{t+1};\vobs_{t+1}\mid\vhid_{t+1})$~\cite{hoffman2016elbo}. Reconstruction-sufficiency with Eq.~\eqref{eq:flow-decomp} gives the lower bound, as $g_t$ is recoverable from $(\vobs_t,\vobs_{t+1})$ and $\vobs_t$ is summarized in $\vhid_{t+1}$. For the KL reduction, the chain rule gives $\Delta_{\KL}=I(\vlat_{t+1};\vdctx_t\mid\vhid_{t+1})-I(\vlat_{t+1};\vdctx_t\mid\vhid_{t+1},\vobs_{t+1})$. The second term vanishes: once $\vobs_{t+1}$ is observed, $g_t$ is determined, and $\vdctx_t$ provides no further information about $\vlat_{t+1}$ whenever $\vdctx_t$ is sufficient for $g_t$; Proposition~\ref{prop:suff} below establishes that training drives $\vdctx_t$ toward exactly this sufficiency. The first term reduces to $I(\vlat_{t+1};g_t\mid\vhid_{t+1})$ by assumption~(v), which ensures that $\vdctx_t$ informs $\vlat_{t+1}$ only through $g_t$. \hfill$\square$
 
\begin{proposition}[Sufficiency of the ego-dynamics prior]
\label{prop:suff}
Under universal approximation of the encoder $\Enc_\xi$ and predictor $\dynpred$, and a persistently exciting window of length $K$ that renders the chassis identifiable from the state--action history, the physics-informed predictor of Section~\ref{subsec:vd-enc}, trained to minimize $\mathcal{L}_{\mathrm{aux}}$ and $\mathcal{L}_{\mathrm{param}}$, makes $\vdctx_t$ sufficient for $g_t$ in the limit of vanishing training loss, i.e.\ $I(g_t;\vobs_{t+1}\mid\vobs_t,\vact_t,\vdctx_t)\to 0$ as $\mathcal{L}_{\mathrm{aux}}\to 0$. The reductions in Propositions~\ref{prop:decomp} and~\ref{prop:bottleneck} are therefore attained.
\end{proposition}
\noindent\textit{Proof.} Under the persistent-excitation assumption~(iv), a window of $K$ consecutive state--action pairs $\mathcal{W}_t$ identifies the chassis context $c$ up to the intrinsic ambiguity of the dynamics, i.e.\ up to dynamics-equivalent parameter combinations~\cite{ljung1999system}; the $\beta$-mixing assumption~(iii) ensures that finite windows are statistically representative of the excitation~\cite{rio2017asymptotic}. The universal-approximation capacity of $\Enc_\xi$ ensures that the mapping from $\mathcal{W}_t$ to $(\vdctx_t,\hat\vparams_t)$ can represent the required inverse identification map. When $\mathcal{L}_{\mathrm{aux}}\to 0$, the predictor $\dynpred$ recovers the true next state $\vstate_{t+1}$ from $(\vdctx_t,\vstate_t,\vact_t)$, which determines $g_t=G(\vstate_t,\vact_t;c)$. Hence $g_t$ becomes a deterministic function of $\vdctx_t$ and the current state--action pair, yielding $I(g_t;\vobs_{t+1}\mid\vobs_t,\vact_t,\vdctx_t)\to 0$. \hfill$\square$
 
\begin{corollary}[Bounded imagination drift]
\label{cor:drift}
During imagination only the prior is rolled out (Section~\ref{sec:prelim}). Without an ego-dynamics prior, the unpredictable ego-motion error compounds over the horizon~$H$. Refreshing $\vdctx_t$ through the online ODE predictor of Section~\ref{subsec:vd-imag} makes $g_t$ predictable from the prior at each step. Each step's ego-motion prediction error then reduces to the ODE predictor error $\epsilon_{\mathrm{ode}}$. If the one-step error propagation of the closed-loop predictor has Lipschitz constant $L$, the cumulative drift over horizon $H$ is bounded by $\mathcal{O}\bigl(\tfrac{L^{H}-1}{L-1}\,\epsilon_{\mathrm{ode}}\bigr)$, which reduces to $\mathcal{O}(H\,\epsilon_{\mathrm{ode}})$ in the non-expansive regime $L\le 1$ that a dynamically stable chassis operating below its handling limit maintains.
\end{corollary}
 
\begin{corollary}[Zero-shot cross-embodiment transfer]
\label{cor:zeroshot}
The ego-dynamics prior $\vdctx_t$ is parameterized through the identifiable chassis parameters $\hat\vparams_t$ rather than memorized appearance. Under the identifiability assumption, an unseen chassis $c'$ yields a re-identified $\hat\vparams_t$ that relocates the prior to the new dynamics regime, enabling the WM and policy to adapt without retraining.
\end{corollary}
 
In summary, the bottleneck is one of representation rather than capacity: ego-motion is determined by the executed action and the chassis, yet the observation-free prior cannot anticipate it, which inflates the prior--posterior KL (Proposition~\ref{prop:bottleneck}). The ego-dynamics conditioning differs from a generic learned context in two respects: it constitutes a physically identifiable sufficient statistic, and it is propagated by a physics-informed ODE predictor. These two properties jointly underpin prior-only imagination (Corollary~\ref{cor:drift}) and zero-shot transfer (Corollary~\ref{cor:zeroshot}).

\section{Experiments}
\label{sec:experiments}
 
\subsection{Experimental Setup}
 
The simulation environment is based on CarDreamer~\cite{CarDreamer2025}, a platform for WM-based autonomous driving built atop the CARLA simulator~\cite{carla2017}. To highlight diverse driving skill learning, we improve CarDreamer by configuring random tasks in Town03 (Urban) and Town04 (Highway) as shown in Fig.~\ref{fig:cardreamer}. Each scenario features background traffic flow governed by the IDM and MOBIL models. The simulation runs at a fixed frequency of 10\,Hz. The agent uses a continuous action configuration $\vact=[a_{\mathrm{acc}},\,a_{\mathrm{steer}}]^\top$ with $a_{\mathrm{acc}}, a_{\mathrm{steer}}\in[-1,1]$. An episode terminates upon collision, timeout (500 steps), out-of-lane violation, or destination arrival. 
 
\begin{figure}[!t]
\centering
\includegraphics[width=0.75\columnwidth]{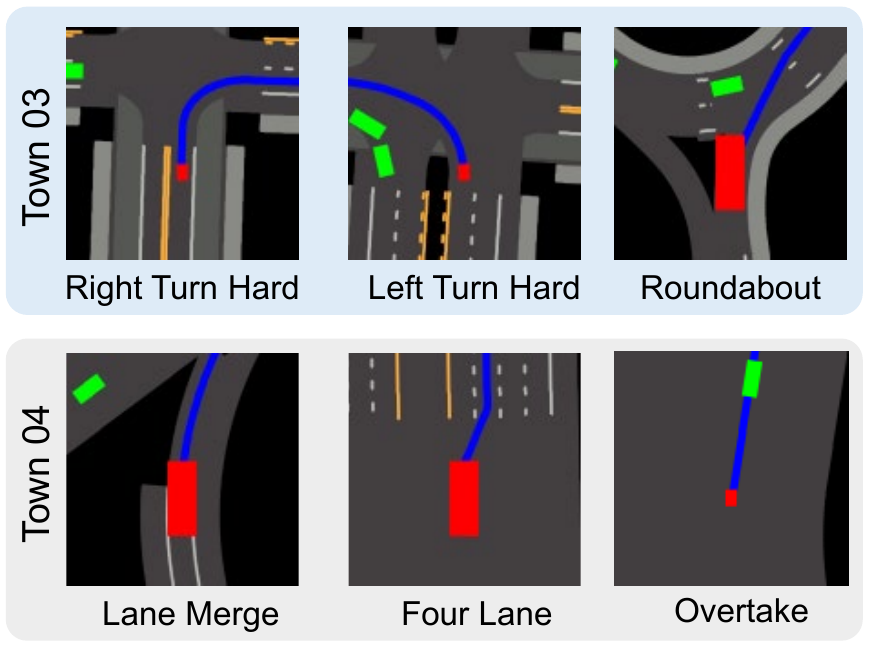}
\caption{Task-diversified CarDreamer: urban scenario in Town03 and highway scenario in Town04. The BEV images demonstrate the WM inputs: four semantic layers are rendered and channel-stacked into a BEV image.}
\label{fig:cardreamer}
\end{figure}

To evaluate generalization across diverse chassis dynamics and test whether the agent captures underlying physical principles, domain randomization in vehicle dynamics during training is applied. As shown in Table~\ref{tab:vehicle_chassis_categories}, the training set comprises 17 vehicle variants spanning four chassis categories (sedan, SUV, van, truck) with distinct control characteristics, while two additional vehicles are reserved for zero-shot evaluation (Fig.~\ref{fig:domain_rand}). This randomization instantiates the chassis-conditioned multimodality of Proposition~\ref{prop:amp}: under an identical policy, dynamically distinct chassis drive the same scene to different temporal evolutions, which the world model must disentangle. To yield a wide distribution of chassis-dynamics features, nine PhysX tire-model parameters are randomized per vehicle instance: multiplicative scale factors drawn from $[0.7,\,1.5]$ are applied to mass, drag coefficient, tire friction, damping rate, maximum steering angle, and lateral and longitudinal stiffness; the center-of-mass (CoM) $x$-offsets and $z$-offsets receive additive perturbations drawn uniformly from $[-0.15,\,0.15]$\,m and $[-0.05,\,0.05]$\,m, respectively. For each vehicle, we extract the ground-truth parameter vector $\vparams^\star$ from the PhysX tire model.

\begin{table}[!t]
\caption{Vehicle Chassis Categories}
\label{tab:vehicle_chassis_categories}
\centering
\footnotesize
\begin{tabularx}{\columnwidth}{@{}
>{\centering\arraybackslash}p{0.9cm}
>{\centering\arraybackslash}p{1.2cm}
>{\centering\arraybackslash}p{5.9cm}
@{}}
\toprule
 & Category & Variant ID \\
\midrule
\multirow{4}{*}{\shortstack{In\\dist.}}
& Sedan  & Coupe, MKZ, Mustang, TT, Model3 \\
& SUV    & Patrol\_21, Patrol, Wrangler, e-tron \\
& Van    & Ambulance, Sprinter, T2, T2\_21 \\
& Truck  & Euro\_HGV, Firetruck, Cybertruck, CarlaCola \\
\midrule
\multirow{2}{*}{\shortstack{Zero\\shot}}
& Bus & Fusorosa \\
& Micro  & Microlino \\
\bottomrule
\end{tabularx}
\end{table}
 
The primary observation is a BEV image of size $128\times128\times3$, containing four semantic layers: road map, planned waypoints, ego vehicle, and surrounding traffic vehicles as shown in Fig.~\ref{fig:cardreamer}. The composite reward at timestep $t$ is:
\begin{align}
r_t &= r_{\mathrm{spd}} + r_{\mathrm{lat}} + r_{\mathrm{hdg}} + r_{\mathrm{wpt}} + r_{\mathrm{stb}} + r_{\mathrm{col}} + r_{\mathrm{dst}}, \nonumber \\
r_{\mathrm{spd}} &= w_v\varphi - w_v\max\left(0,\tfrac{v_\parallel - v_{\mathrm{des}}}{v_{\mathrm{des}}}\right) - \tfrac{1}{2}\min\left(v_\perp,\,1\right), \nonumber \\
r_{\mathrm{lat}} &= w_d\,\varphi\,\exp\left(-d_\perp^2/\sigma^2\right), \nonumber \\
r_{\mathrm{hdg}} &= w_\theta\bigl[\varphi\max\left(\cos\Delta\theta,\,0\right) + \min\left(\cos\Delta\theta,\,0\right)\bigr], \nonumber \\
r_{\mathrm{wpt}} &= w_{\mathrm{wpt}}\,\mathbf{1}[\Delta n_{\mathrm{wpt}}>0], \label{eq:reward} \\
r_{\mathrm{stb}} &= -w_s\bigl[\max\left(|a_y|-1.5,\,0\right) + 0.05\max\left(|\dot\omega|-8,\,0\right)\bigr], \nonumber \\
r_{\mathrm{col}} &= -w_c\max\left(v,\,1\right)\,\mathbf{1}[\text{collision}], \nonumber \\
r_{\mathrm{dst}} &= w_{\mathrm{dst}}\,\mathbf{1}[\text{goal reached}]. \nonumber
\end{align}
where $\varphi = \mathrm{clip}(v_\parallel/v_{\mathrm{des}},\,0,\,1)$ is the speed-gating factor, with $v_\parallel$ the path-tangential ego speed, $v_\perp$ the lateral ego speed, and $v_{\mathrm{des}}$ the target speed; $d_\perp$ is the cross-track deviation; $\Delta\theta$ the heading error; $\Delta n_{\mathrm{wpt}}$ the number of waypoints completed in the step; $a_y$ the lateral acceleration; and $\dot\omega$ the yaw acceleration.
Weights: $w_v{=}2.0$, $w_d{=}1.0$, $w_\theta{=}0.5$, $w_{\mathrm{wpt}}{=}2.0$, $w_s{=}0.02$, $w_c{=}30$, $w_{\mathrm{dst}}{=}20$; $\sigma{=}1.5$\,m.
The speed gate $\varphi$ ensures a stationary vehicle receives zero dense reward, preventing exploitation without forward motion; the dead zone in $r_{\mathrm{stb}}$ ($|a_y|<1.5$\,m/s$^2$, $|\dot\omega|<8$\,rad/s$^2$) permits normal cornering without penalty.

\begin{figure}[!t]
\centering
\includegraphics[width=\columnwidth]{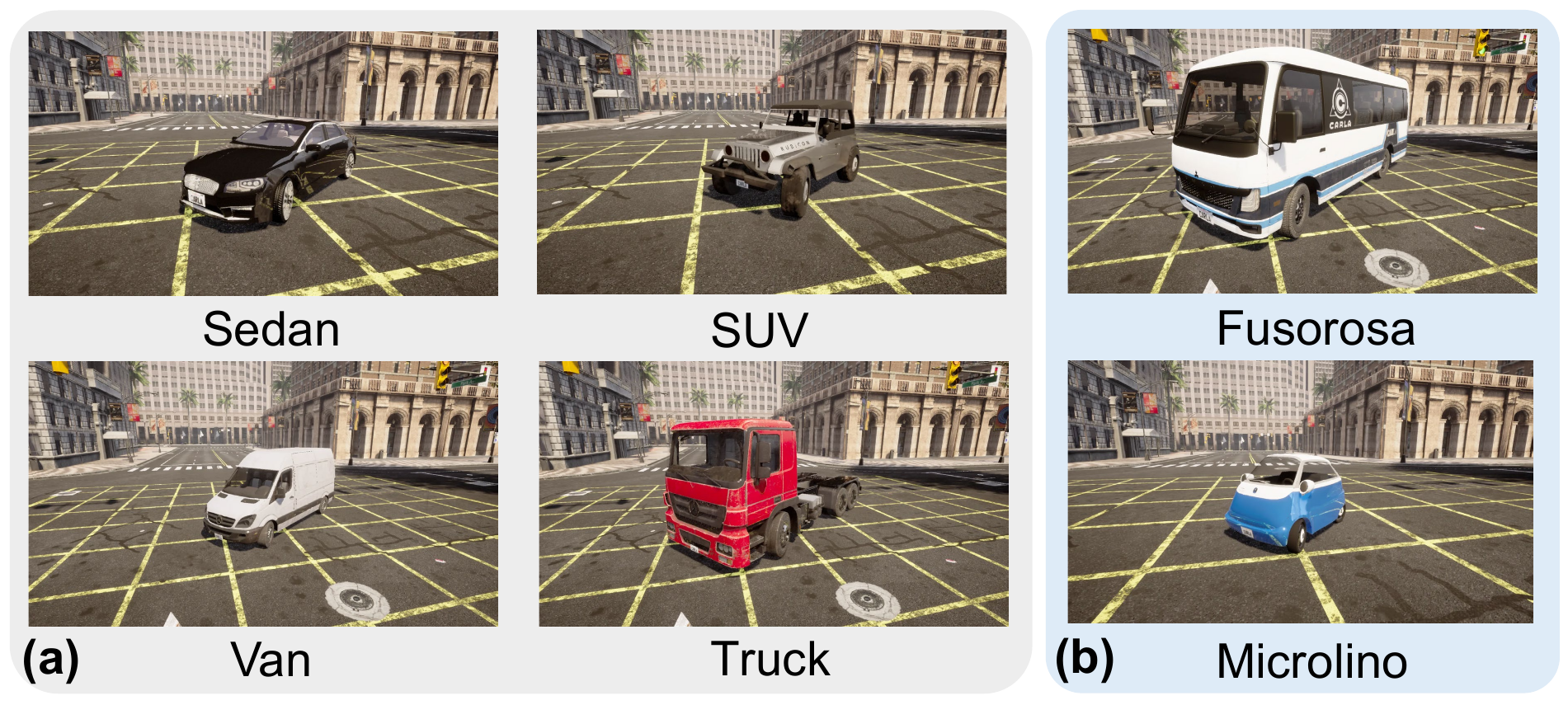}
\caption{Chassis randomization. (a) Training vehicles span four categories, with per-instance chassis parameter randomization. (b) Two unseen vehicle variants (Fusorosa, Microlino) are used for zero-shot evaluation.}
\label{fig:domain_rand}
\end{figure}
 
\subsection{Training Details}
 
We implement DynaDreamer on STORM~\cite{STORM2023}, and compare DynaDreamer against baselines and ablation variants.
 
\subsubsection{Baselines}
(1)~\textit{SAC} (Soft Actor-Critic)~\cite{softActorCritic2018}: a model-free RL method with a convolutional neural network (CNN) encoder.
(2)~\textit{DreamerV3}~\cite{Dreamer2025}: a model-based method using a recurrent state-space model (RSSM) as its WM.
(3)~\textit{STORM}~\cite{STORM2023}: a Transformer-based WM, and an actor-critic agent trained from imagined rollouts.

\subsubsection{Ablation and injection variants}
Each variant modifies one design aspect of Section~\ref{sec:method} to isolate its contribution.
(1)~\textit{VD-STORM}: an injection variant that equips STORM with the same VD-context and ego-dynamics modules as DynaDreamer, but feeds $\vdctx_t$ only to the actor-critic while leaving the WM unconditioned. This variant is a controlled re-implementation of policy-level physical conditioning~\cite{OneModeltoDriftThemAll} on the shared backbone, and isolates the effect of the injection point on the prior--posterior bottleneck.
(2)~\textit{w/o Neural-ODE}: replaces the physics-informed dynamics predictor (bicycle ODE, neural tire model, and neural residual) with a pure residual MLP, and removes the parameter estimation head together with the physics-informed features, thereby eliminating the physical identifiability mechanism. The resulting purely learned context is a controlled counterpart of generic context-based adaptation~\cite{MetaRL} on the shared backbone.
(3)~\textit{w/o Modulation}: replaces AdaLN conditioning with concatenation of $\vdctx_t$ at the Transformer input stem and substitutes absolute positional encoding for RoPE, weakening the distributional conditioning pathway.
(4)~\textit{w/o Multi-step}: reduces the auxiliary prediction horizon from $K_{\mathrm{pred}}{=}5$ (autoregressive multi-step) to $K_{\mathrm{pred}}{=}1$ (single-step), degrading the predictor fidelity required for imagination-time synchronization.
 
\subsubsection{Evaluation protocol}
Each method is evaluated over 300 episodes partitioned into three conditions. \textit{S1} (in-distribution, 100 episodes): vehicles are randomly drawn from the 17 vehicle variant training pool. \textit{S2} (zero-shot, 100 episodes): the unseen Fusorosa, a bus-class vehicle whose 5.63\,m wheelbase exceeds that of every training variant. \textit{S3} (zero-shot, 100 episodes): the unseen Microlino, a micro-class vehicle whose wheelbase (1.52\,m) and mass are both below those of any training variant. The two vehicles bracket the training wheelbase range (2.31--5.0\,m) from both sides; together with the 17 training variants, the evaluation spans microcar, sedan, SUV, van, truck, and bus classes. S2 and S3 therefore probe not only whether the ego-dynamics prior generalizes to handling regimes outside the training distribution, but whether a single WM can serve fleet-level deployment without per-platform retraining. We report 5 categories of metrics:
\begin{itemize}
\item Task completion: success rate, collision rate, out-of-lane rate, timeout rate, cumulative reward.
\item Driving quality: heading error, lateral deviation, average jerk, average speed.
\item BEV reconstruction: peak signal-to-noise ratio (PSNR), mean absolute error (MAE).
\item Low-dimensional state prediction: MAE at different horizons ($H{=}1,3,5$) and per-component breakdown. 
\item Chassis parameter estimation: relative error mean $\bigl(\frac{1}{7}\sum_{k=1}^{7}|[\hat{\vparams}]_k - [\vparams^\star]_k|/|[\vparams^\star]_k|\bigr)$, log-space mean squared error (MSE) $\bigl(\frac{1}{7}\sum_{k=1}^{7}(\log[\hat{\vparams}]_k - \log[\vparams^\star]_k)^2\bigr)$, and consistency (standard deviation of relative errors across episodes).
\end{itemize}
\subsubsection{Implementation details}
Table~\ref{tab:training_hyper} summarizes the training hyperparameters. All STORM-based methods share the same backbone architecture. On Town03, all methods are trained for $10^5$ real-interaction steps; on Town04, $5\times10^4$ steps suffice as the simpler highway geometry leads to faster convergence. All methods use a replay buffer of $10^5$ transitions and a warm-up of 1024 random steps. Three random seeds are used for reproduction.

Fig.~\ref{fig:training_curves} presents the training curves on Town03 (Fig.~\ref{fig:training_curves}(a)-(b)) and Town04 (Fig.~\ref{fig:training_curves}(c)-(d)). SAC fails to converge on both Town03 and Town04 and is excluded from subsequent evaluation. On Town03, DynaDreamer converges to the highest episode score and control score across all remaining baselines and ablation variants, demonstrating stable and efficient learning throughout training. On Town04, DynaDreamer again attains the highest episode score and control score across all compared methods.
 
\begin{figure}[!t]
\centering
\includegraphics[width=0.9\columnwidth]{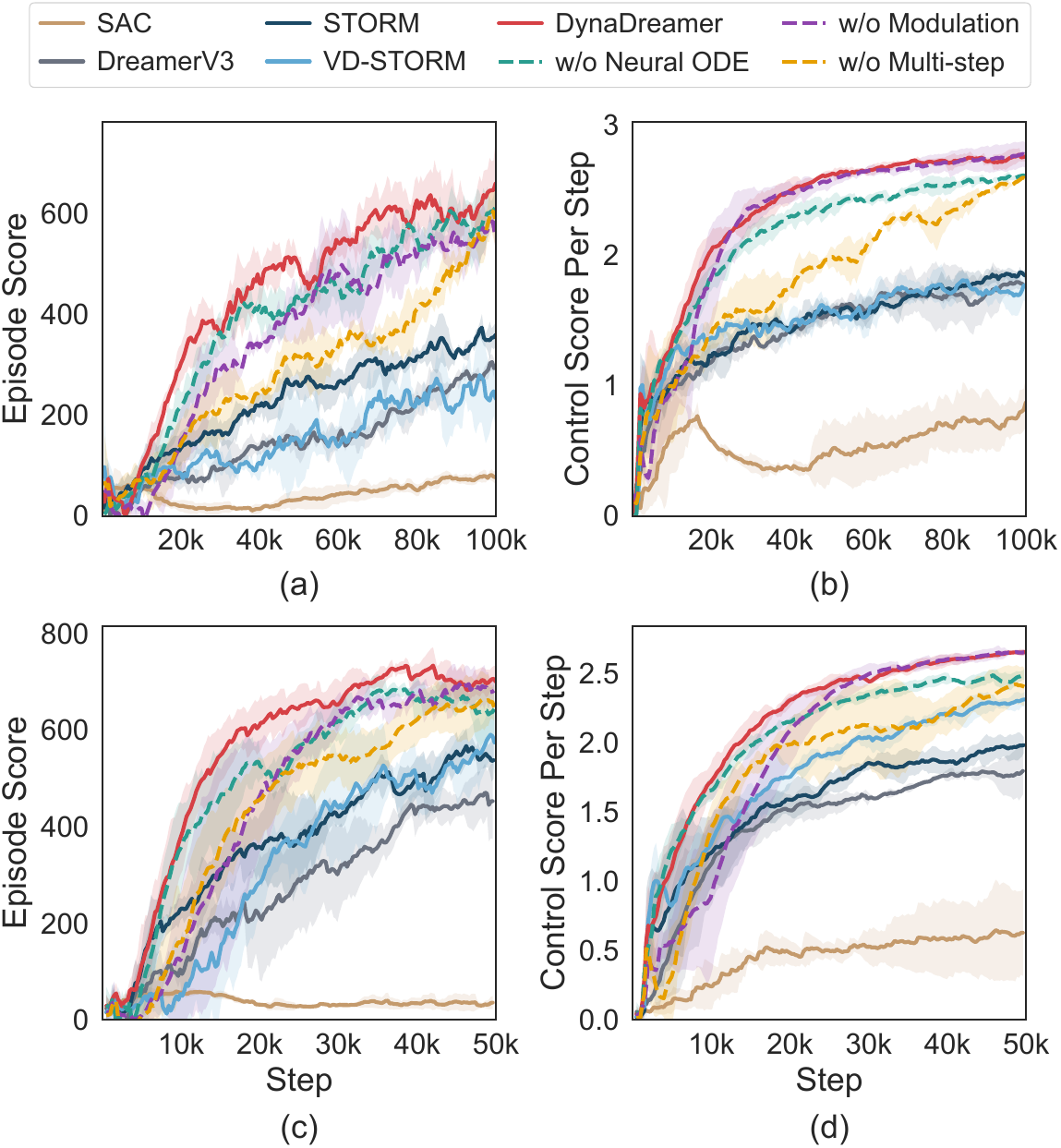}
\caption{Training curves on Town03 (Urban, a-b) and Town04 (Highway, c-d). Solid lines: baselines and the injection variant VD-STORM; dashed lines: ablation variants. Shaded regions denote standard deviation across three seeds. }
\label{fig:training_curves}
\end{figure}
 
\begin{table}[h!]
\caption{Training Hyperparameters}
\label{tab:training_hyper}
\centering
\footnotesize
\begin{tabularx}{\columnwidth}{@{}
>{\hsize=0.90\hsize\centering\arraybackslash}X
>{\hsize=1.50\hsize\centering\arraybackslash}X
>{\hsize=0.60\hsize\centering\arraybackslash}X
@{}}
\toprule
Hyperparameter & Description & Value \\
\midrule
$\mathrm{lr}_{\mathrm{wm}}$ & WM learning rate & $1{\times}10^{-4}$ \\
$\mathrm{lr}_{\mathrm{agent}}$ & Agent learning rate & $3{\times}10^{-5}$ \\
$B \times T$ & Batch size (seq.\ $\times$ length) & $16 \times 64$ \\
$|\mathcal{D}|$ & Replay buffer capacity & $10^{5}$ \\
$\gamma$ & Discount factor & 0.99 \\
$\lambda$ & Trace-decay coefficient & 0.95 \\
$H$ & Imagination horizon & 16 \\
$K_{\mathrm{ctx}}$ & Imag.\ warm-up context & 8 \\
$\beta_{\mathrm{ent}}$ & Entropy coefficient & $3{\times}10^{-4}$ \\
$K_{\mathrm{pred}}$ & Aux.\ prediction steps & 5 \\
$\beta_{\mathrm{dyn}}$ / $\beta_{\mathrm{rep}}$ & KL loss weights & 0.5 / 0.1 \\
$\alpha_{\mathrm{aux}}$ / $\alpha_{\mathrm{param}}$ & Aux.\ / param.\ loss weights & 1.0 / 0.1 \\
\bottomrule
\end{tabularx}
\end{table}
 
\begin{figure}[!t]
\centering
\includegraphics[width=0.9\columnwidth]{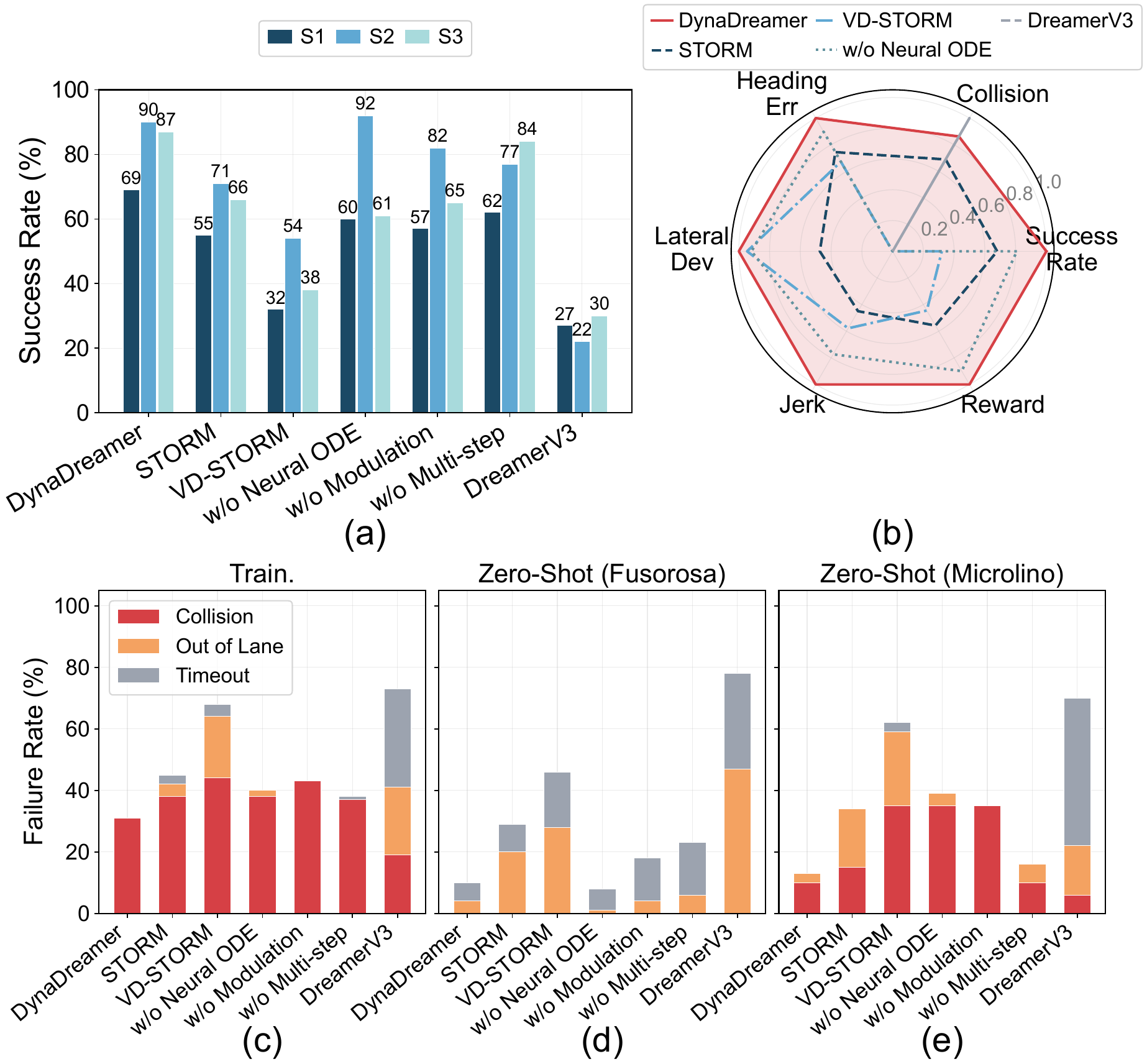}
\caption{(a) Success rates across three evaluation conditions. (b) Radar chart normalizing six performance dimensions to $[0,1]$. (c)-(e) Failure mode breakdown for S1, S2, and S3, respectively.}
\label{fig:success_failure}
\end{figure}
 
\begin{table*}[!t]
\caption{Task Performance Comparison on Town03 (Urban) and Town04 (Highway)}
\label{tab:main_results}\label{tab:town04_results}
\centering
\setlength{\tabcolsep}{3.2pt}
\renewcommand{\arraystretch}{1.05}
\begin{tabular}{p{1.4cm} l ccccc|ccccc|ccccc}
\toprule
& & \multicolumn{5}{c}{Training-Identical (S1)} & \multicolumn{5}{c}{Zero-Shot Fusorosa (S2)} & \multicolumn{5}{c}{Zero-Shot Microlino (S3)} \\
\cmidrule(lr){3-7} \cmidrule(lr){8-12} \cmidrule(lr){13-17}
Scenario & Method & Succ.$\uparrow$ & Coll.$\downarrow$ & OOL$\downarrow$ & T.O.$\downarrow$ & Rew.$\uparrow$
       & Succ.$\uparrow$ & Coll.$\downarrow$ & OOL$\downarrow$ & T.O.$\downarrow$ & Rew.$\uparrow$
       & Succ.$\uparrow$ & Coll.$\downarrow$ & OOL$\downarrow$ & T.O.$\downarrow$ & Rew.$\uparrow$ \\
\midrule
\multirow{7}{1.4cm}{\raggedright\textit{Town03 (Urban)}}
& DreamerV3      & 0.27 & \textbf{0.19} & 0.22 & 0.32 & 277 & 0.22 & 0.00 & 0.47 & 0.31 & 338 & 0.30 & \textbf{0.06} & 0.16 & 0.48 & 387 \\
& STORM          & 0.55 & 0.38 & 0.04 & 0.03 & 337 & 0.71 & 0.00 & 0.20 & 0.09 & 540 & 0.66 & 0.15 & 0.19 & 0.00 & 623 \\
& VD-STORM       & 0.32 & 0.44 & 0.20 & 0.04 & 219 & 0.54 & 0.00 & 0.28 & 0.18 & 537 & 0.38 & 0.35 & 0.24 & 0.03 & 537 \\
& w/o Neural-ODE & 0.60 & 0.38 & 0.02 & 0.00 & 586 & \textbf{0.92} & 0.00 & \textbf{0.01} & 0.07 & 794 & 0.61 & 0.35 & 0.04 & 0.00 & 637 \\
& w/o Modulation & 0.57 & 0.43 & 0.00 & 0.00 & 435 & 0.82 & 0.00 & 0.04 & 0.14 & 717 & 0.65 & 0.35 & \textbf{0.00} & 0.00 & 564 \\
& w/o Multi-step & 0.62 & 0.37 & 0.00 & 0.01 & 584 & 0.77 & 0.00 & 0.06 & 0.17 & 759 & 0.84 & 0.10 & 0.06 & 0.00 & 678 \\
& \textbf{DynaDreamer} & \textbf{0.69} & 0.31 & \textbf{0.00} & \textbf{0.00} & \textbf{604} & 0.90 & \textbf{0.00} & 0.04 & \textbf{0.06} & \textbf{816} & \textbf{0.87} & 0.10 & 0.03 & \textbf{0.00} & \textbf{694} \\
\midrule
\multirow{7}{1.4cm}{\raggedright\textit{Town04 (Highway)}}
& DreamerV3      & 0.53 & 0.31 & 0.15 & 0.01 & 260 & 0.77 & 0.00 & 0.15 & 0.08 & 313 & 0.59 & 0.24 & 0.07 & 0.10 & 516 \\
& STORM          & 0.56 & 0.22 & 0.01 & 0.21 & 516 & 0.55 & 0.00 & 0.37 & 0.08 & 524 & 0.57 & 0.08 & \textbf{0.00} & 0.35 & 689 \\
& VD-STORM       & 0.78 & 0.16 & \textbf{0.00} & 0.06 & 554 & 0.87 & 0.00 & 0.07 & 0.06 & 466 & 0.87 & 0.05 & 0.00 & 0.08 & 782 \\
& w/o Neural-ODE & 0.92 & 0.07 & 0.01 & \textbf{0.00} & \textbf{776} & \textbf{0.97} & 0.00 & 0.00 & \textbf{0.03} & \textbf{737} & 0.86 & \textbf{0.02} & 0.12 & 0.00 & 812 \\
& w/o Modulation & 0.91 & 0.09 & 0.00 & 0.00 & 748 & 0.97 & 0.00 & 0.00 & 0.03 & 730 & \textbf{0.95} & 0.05 & 0.00 & 0.00 & \textbf{823} \\
& w/o Multi-step & 0.91 & 0.05 & 0.04 & 0.00 & 647 & 0.91 & 0.00 & 0.03 & 0.06 & 625 & 0.84 & 0.16 & 0.00 & 0.00 & 699 \\
& \textbf{DynaDreamer} & \textbf{0.93} & \textbf{0.04} & 0.02 & 0.01 & 753 & 0.95 & \textbf{0.00} & \textbf{0.00} & 0.05 & 718 & 0.82 & 0.15 & 0.03 & \textbf{0.00} & 739 \\
\bottomrule
\end{tabular}
\begin{flushleft}
\footnotesize Succ.=Success Rate, Coll.=Collision, OOL=Out-of-Lane, T.O.=Timeout, Rew.=Cumulative Reward. Bold = best per column within each section. SAC excluded (non-convergence).
\end{flushleft}
\end{table*}
 
\subsection{Task Performance and Driving Quality}
 
Table~\ref{tab:main_results} summarizes results across both scenarios and three evaluation conditions. On Town03, DynaDreamer achieves the highest average success rate across S1--S3 (82.0\%), a 28\% relative improvement over STORM, while maintaining a low collision rate. VD-STORM, which receives the same ego-dynamics context but only at the policy level, degrades below STORM (41.3\% vs.\ 64.0\%): the actor receives chassis information, yet its policy is optimized within chassis-agnostic imagined rollouts, creating a train--test discrepancy. The ego-dynamics information must therefore enter the WM's predictive distributions, the bottleneck identified by Proposition~\ref{prop:bottleneck}. On Town04, DynaDreamer improves over the strongest baseline (DreamerV3) by 43\% on average; relative to the shared Transformer WM backbone (STORM), the advantage reaches 73\% on the unseen bus-class chassis (S2). Simplified variants match or exceed the full model in several highway cells (w/o~Neural-ODE on S1--S2, w/o~Modulation on S3), because the low-curvature geometry weakly excites lateral dynamics and reduces the ego-motion burden. This is consistent with Proposition~\ref{prop:decomp}: the burden concentrates in the lateral and yaw components (cf.\ Fig.~\ref{fig:lowdim_comp}), and the full architecture yields its largest gains in the urban scenario, where these components are strongly excited. The ego-dynamics conditioning further transfers to the unseen Fusorosa and Microlino without retraining, consistent with Corollary~\ref{cor:zeroshot}. The subsequent analysis focuses on Town03, where the urban geometry better differentiates the methods.

A pivotal comparison concerns the physical grounding itself. The purely learned context of w/o~Neural-ODE stays close to the full model within the training support (0.60 vs.\ 0.69 success on Town03~S1; 0.92 vs.\ 0.93 on Town04~S1) and remains competitive on the long-wheelbase Fusorosa (S2), whose sluggish yaw response only mildly excites the lateral channel. It collapses, however, on the sub-tonne Microlino in the urban scenario: 0.61 success with a 0.35 collision rate, against 0.87 and 0.10 for DynaDreamer. The Microlino lies below the training range in both mass and wheelbase, and its high yaw gain maximizes the ego-induced warp per unit steering (Proposition~\ref{prop:amp}); a generic learned context cannot relocate its conditioning to this unseen dynamics regime, whereas the physics-grounded $\vdctx_t$ re-identifies the chassis online and preserves performance (Corollary~\ref{cor:zeroshot}). Physical identifiability thus converts in-distribution competence into extrapolative robustness at the far end of the fleet spectrum.

Figs.~\ref{fig:success_failure} and~\ref{fig:driving_metrics} provide complementary breakdowns of failure modes and driving quality on Town03. DynaDreamer exhibits balanced, low failure rates across all modes, whereas STORM and VD-STORM suffer from high out-of-lane rates (19.5\% and 26.0\%, respectively); ablation variants w/o~Modulation and w/o~Neural-ODE show elevated collision rates (17.5\%), confirming that ego-dynamics augmentation is critical for safety. Among the three components, adaptive modulation has the strongest individual effect on task success rate. For driving quality, DynaDreamer achieves the lowest heading error and jerk, reflecting the smoothest trajectory following, and the radar chart in Fig.~\ref{fig:success_failure}(b) confirms its largest enclosed area, indicating the most well-balanced overall performance.
 
\begin{figure}[!t]
\centering
\includegraphics[width=0.9\columnwidth]{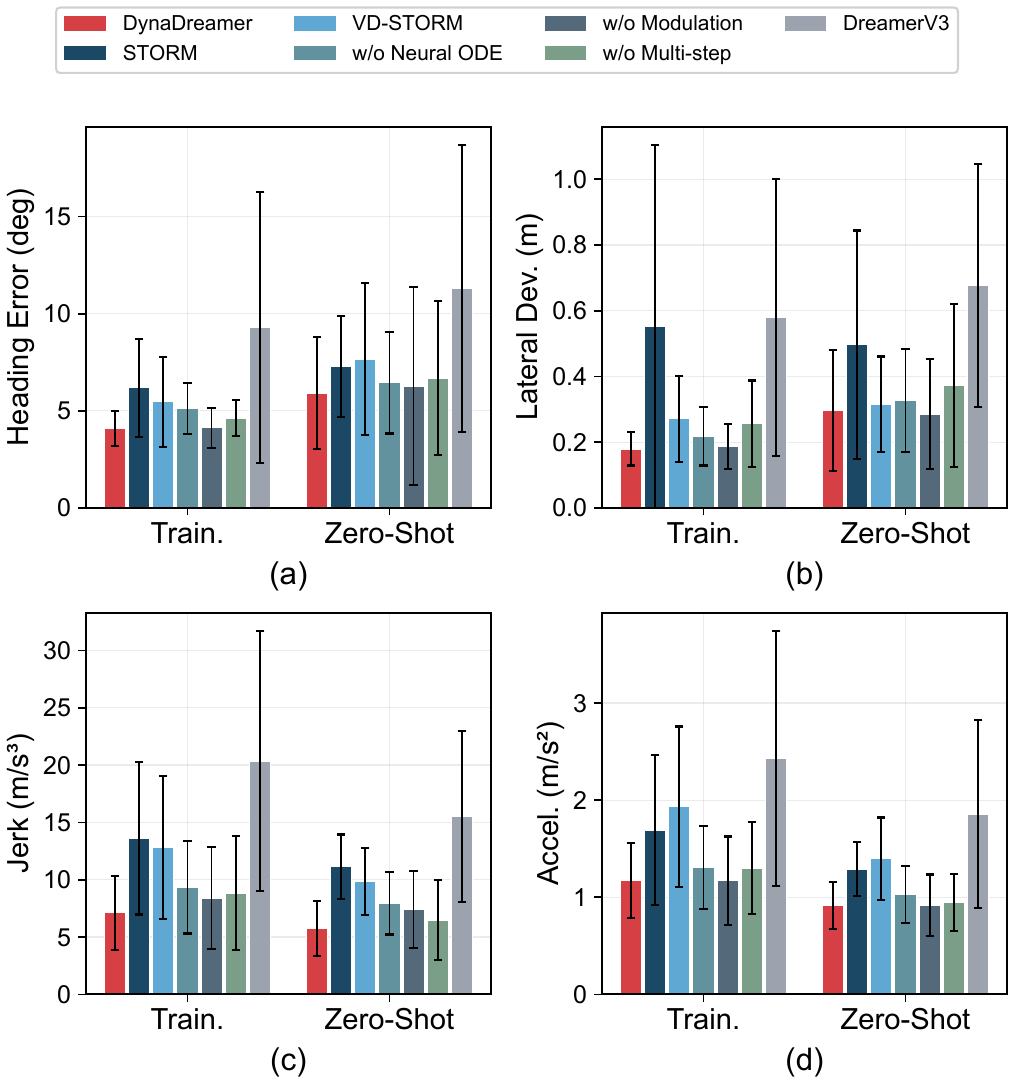}
\caption{Comparison of lateral deviation, heading error, jerk, and speed across methods (Pooled Zero-Shot, S2+S3).}
\label{fig:driving_metrics}
\end{figure}
 
\subsection{World Model Quality}
 
Table~\ref{tab:bev} reports WM quality under pooled zero-shot evaluation (S2+S3), combining prior--posterior alignment (KL divergence) and BEV reconstruction (PSNR/MAE). DynaDreamer attains the lowest KL (10.62, vs.\ 36.80 for DreamerV3 and 14.19 for STORM), the highest PSNR, and the lowest MAE, indicating both well-calibrated latent dynamics and faithful image generation. Ablation variants achieve comparable PSNR but consistently higher KL than DynaDreamer, suggesting that each component contributes to tightening the prior--posterior gap. Evaluated on unseen chassis (S2+S3), where the prior's inability to anticipate ego-motion is most acute, this KL gap is the central empirical signature of Proposition~\ref{prop:bottleneck}. VD-STORM injects the ego-dynamics context only into the actor-critic while leaving the prior unconditioned, and attains a KL of 15.80, no better than STORM (14.19); routing the same context into the world model's prior and posterior instead lowers the KL to 10.62. This confirms that the availability of ego-dynamics information alone is insufficient: the bottleneck of Proposition~\ref{prop:bottleneck} contracts only when this information directly conditions the WM's observation-free prior. The monotonic decrease from STORM (14.19) through w/o~Modulation (13.54, partial conditioning) to DynaDreamer (10.62, full AdaLN conditioning) further corroborates that the KL reduction scales with the degree to which the prior is informed of ego-dynamics.
 
\begin{table}[!t]
\caption{World Model Quality: Pooled Zero-Shot Evaluation (S2+S3).}
\label{tab:bev}
\centering
\resizebox{\columnwidth}{!}{%
\begin{tabular}{l ccc}
\toprule
Method & KL$_{\mathrm{prior\|post}}$ $\downarrow$ & BEV PSNR (dB) $\uparrow$ & BEV MAE $\downarrow$ \\
\midrule
DreamerV3      & 36.80 & 20.88 & 0.0337 \\
STORM          & 14.19 & 22.53 & 0.0347 \\
VD-STORM       & 15.80 & 23.24 & 0.0310 \\
w/o Neural-ODE & 13.67 & 23.24 & 0.0282 \\
w/o Modulation & 13.54 & 22.70 & 0.0345 \\
w/o Multi-step & 11.18 & 23.29 & 0.0298 \\
 
\textbf{DynaDreamer} & \textbf{10.62} & \textbf{23.55} & \textbf{0.0269} \\
\bottomrule
\end{tabular}}
\end{table}
 
Figs.~\ref{fig:rollout_corr}--\ref{fig:lowdim_comp} further characterize imagination quality. DynaDreamer sustains the highest rollout reward correlation as the horizon grows. Its prediction error grows slowly from H1 to H5, whereas w/o~Multi-step exhibits error explosion. This contrast highlights the practical importance of the drift bound in Corollary~\ref{cor:drift}: multi-step supervision keeps $\epsilon_{\mathrm{ode}}$ small, and refreshing $\vdctx_t$ along the rollout prevents ego-motion error from compounding. The per-component breakdown at H5 shows DynaDreamer excelling at lateral dynamics prediction, directly supporting lane-keeping performance. This advantage concentrates on the lateral and yaw components, which Eq.~\eqref{eq:flow-decomp} attributes to the ego-induced warp. The burden isolated in Proposition~\ref{prop:decomp} is therefore precisely the one the prior offloads.
 
\begin{figure}[!t]
\centering
\includegraphics[width=0.9\columnwidth]{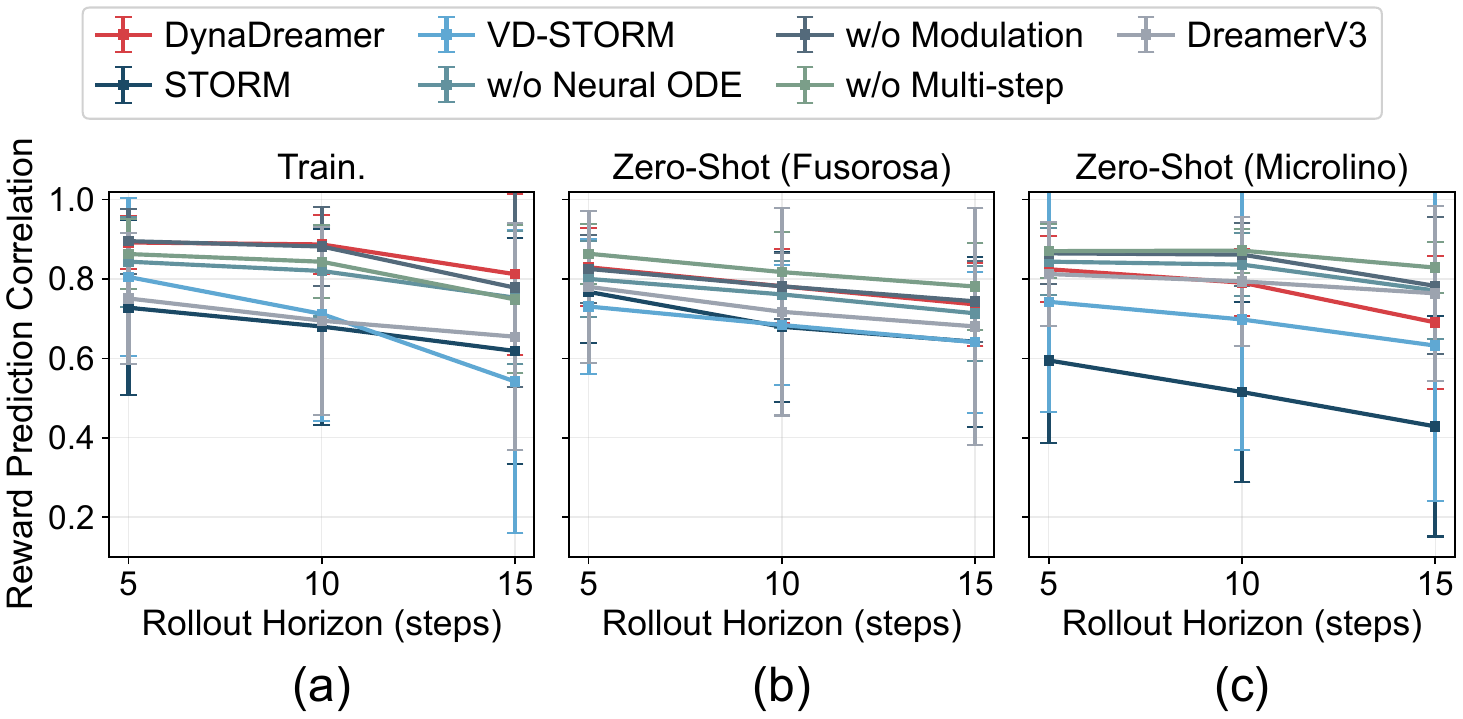}
\caption{Rollout reward correlation at horizons H5, H10, H15 across three evaluation conditions.}
\label{fig:rollout_corr}
\end{figure}
 
\begin{figure}[!t]
\centering
\includegraphics[width=0.9\columnwidth]{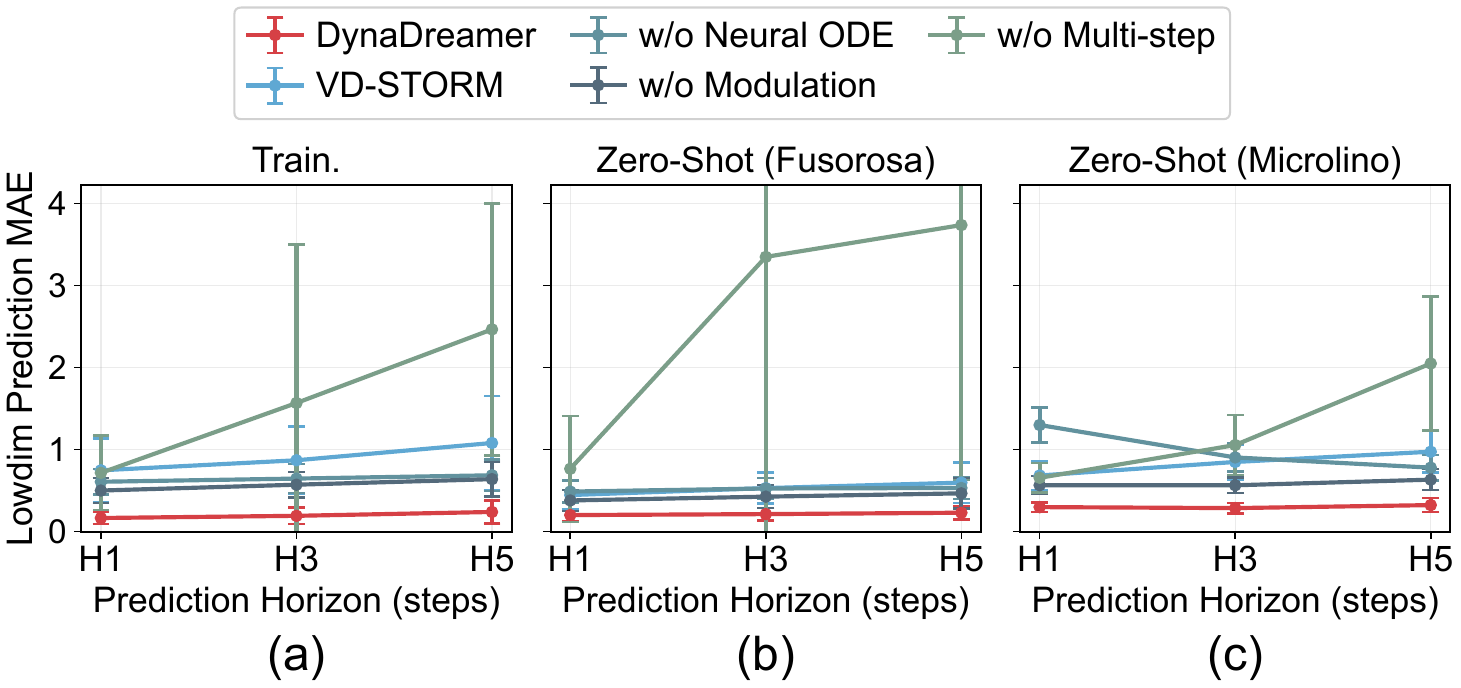}
\caption{Low-dimensional state prediction: overall prediction MAE at horizons H1, H3, H5 under the training-identical and the two zero-shot conditions.}
\label{fig:lowdim_horizon}
\end{figure}
 
\begin{figure}[!t]
\centering
\includegraphics[width=0.9\columnwidth]{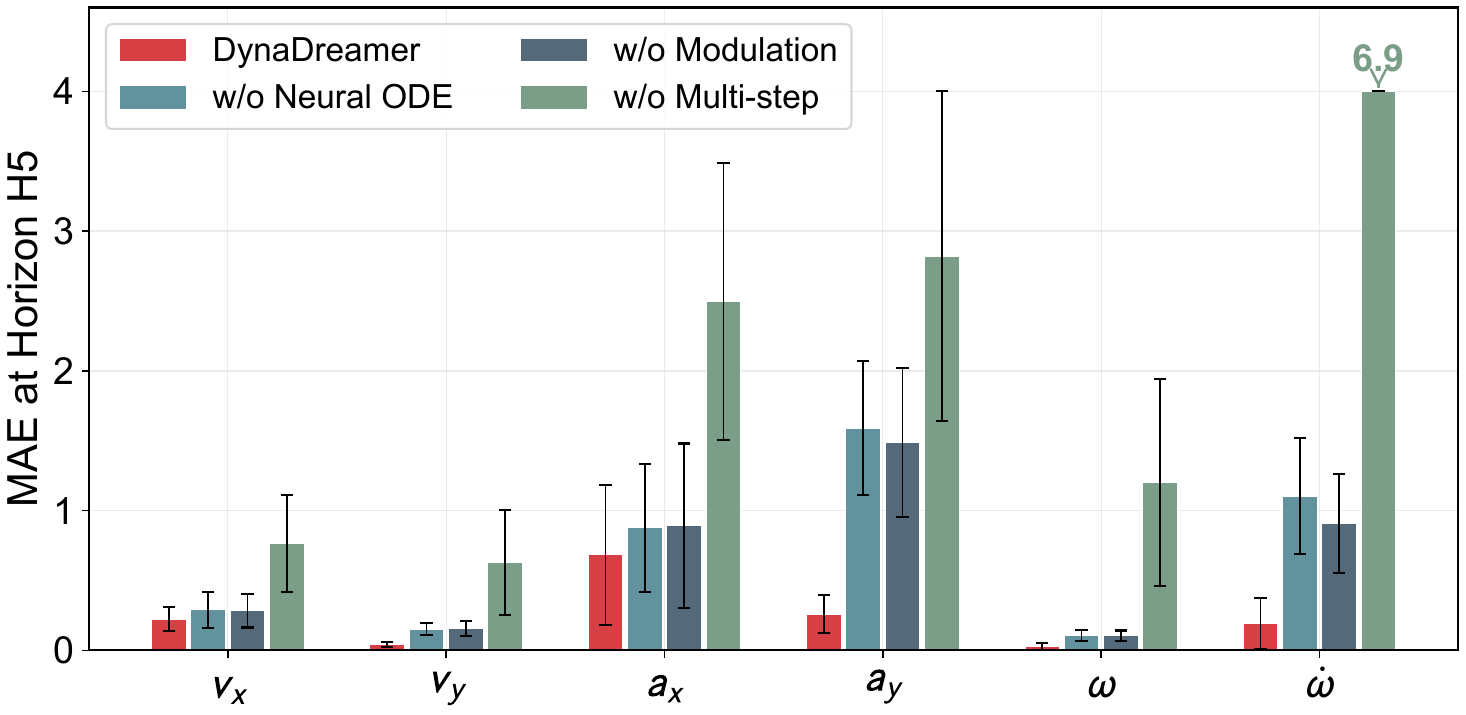}
\caption{Per-component prediction MAE at H5, covering $v_x$, $v_y$, $a_x$, $a_y$, yaw rate, and yaw acceleration. Training-identical condition.}
\label{fig:lowdim_comp}
\end{figure}
 
Proposition~\ref{prop:suff} establishes sufficiency with respect to the ego motion $g_t$, which depends on parameter combinations within the bicycle model. Individual parameters are therefore identifiable only up to dynamically equivalent combinations, an intrinsic ambiguity acknowledged in the proposition. The appropriate test of sufficiency is consequently the ego-state prediction error on unseen chassis. Table~\ref{tab:lowdim_zeroshot} reports the pooled zero-shot prediction MAE (disaggregated per vehicle in Fig.~\ref{fig:lowdim_horizon}): DynaDreamer attains the lowest error at every horizon and, crucially, its error stays nearly flat from H1 to H5 (0.250 to 0.277), the empirical signature of the bounded drift of Corollary~\ref{cor:drift}, whereas w/o~Multi-step exhibits an order-of-magnitude explosion. Point estimates of the parameters themselves are a stricter and partially ill-posed target: under zero-shot evaluation they carry large relative errors for all methods (Table~\ref{tab:param_summary}), exactly as the parameter-coupling ambiguity predicts. DynaDreamer still achieves the lowest error and best consistency, while w/o~Multi-step degrades noticeably, indicating that multi-step auxiliary supervision provides gradient signals that also improve parameter learning; the neural residual of Eq.~\eqref{eq:dynpred} absorbs the remaining parameter bias, as the case study below illustrates. Table~\ref{tab:param_detail} provides a per-parameter breakdown under training-identical conditions, where relative errors range from 0.090 (steering limit) to 0.509 (moment of inertia). The higher error on $I_z$ is expected given its large variance across vehicle types and its indirect observability from driving behavior alone.

\begin{table}[!t]
\caption{Zero-Shot Ego-State Prediction MAE (Pooled S2+S3).}
\label{tab:lowdim_zeroshot}
\centering
\setlength{\tabcolsep}{5pt}
\begin{tabular}{l ccc}
\toprule
Method & H1 $\downarrow$ & H3 $\downarrow$ & H5 $\downarrow$ \\
\midrule
VD-STORM       & 0.565 & 0.690 & 0.786 \\
w/o Neural-ODE & 0.894 & 0.715 & 0.656 \\
w/o Modulation & 0.473 & 0.495 & 0.550 \\
w/o Multi-step & 0.711 & 2.203 & 2.894 \\
\textbf{DynaDreamer} & \textbf{0.250} & \textbf{0.250} & \textbf{0.277} \\
\bottomrule
\end{tabular}
\end{table}
 
\begin{table}[!t]
\caption{Chassis Parameter Estimation: Method Comparison (Pooled Zero-Shot, S2+S3).}
\label{tab:param_summary}
\centering
\setlength{\tabcolsep}{3pt}
\begin{tabular}{l ccc}
\toprule
Method & Rel. Err. $\downarrow$ & Log MSE $\downarrow$ & Consist. $\downarrow$ \\
\midrule
VD-STORM       & 1.527 & 0.910 & 1.019 \\
w/o Multi-step & 1.838 & 1.080 & 1.043 \\
\textbf{DynaDreamer} & \textbf{1.481} & \textbf{0.868} & \textbf{0.844} \\
\bottomrule
\end{tabular}
\end{table}
 
\begin{table}[!t]
\caption{DynaDreamer Per-Parameter Estimation Breakdown (Training-Identical).}
\label{tab:param_detail}
\centering
\setlength{\tabcolsep}{3pt}
\begin{tabular}{l ccc}
\toprule
Parameter & Ground Truth & Predicted & Rel. Err. \\
\midrule
Mass (kg)                   & 3034 $\pm$ 2710 & 2955 $\pm$ 1759 & 0.329 \\
$I_z$ (kg$\cdot$m$^2$)      & 3700 $\pm$ 6039 & 3437 $\pm$ 3337 & 0.509 \\
$C_f$ (kN/rad)              & 93.3 $\pm$ 62.2 & 80.0 $\pm$ 38.6 & 0.331 \\
$C_r$ (kN/rad)              & 51.1 $\pm$ 27.5 & 40.2 $\pm$ 16.4 & 0.255 \\
$l_f$ (m)                   & 1.17 $\pm$ 0.37 & 1.12 $\pm$ 0.24 & 0.153 \\
$l_r$ (m)                   & 2.07 $\pm$ 0.56 & 2.25 $\pm$ 0.43 & 0.120 \\
$\delta_{\max}$ ($^\circ$)  & 70.0 $\pm$ 5.5  & 64.4 $\pm$ 0.6  & 0.090 \\
\bottomrule
\end{tabular}
\end{table}

\subsection{Case Study}
 
We present a case study on a zero-shot rollout of the 706~kg Microlino subcompact (below the 1105~kg minimum of the training fleet) and contrast STORM with DynaDreamer on a roundabout task. A supplementary Fusorosa bus rollout probes the opposite, long-wheelbase end of the fleet spectrum in the parameter-identification case.
 
\subsubsection{Qualitative Trajectory and Control Behavior}
 
Fig.~\ref{fig:case_traj} plots signed lateral deviation and heading error. STORM remains biased to one side of the lane for 71\% of the episode with a $-1.45$~m peak excursion, whereas DynaDreamer oscillates symmetrically around the center line and stays within $0.70$~m; heading error follows the same pattern, swinging up to $\pm 45^\circ$ for STORM versus $\pm 16^\circ$ for DynaDreamer. This tighter tracking yields the smoother control signals of Fig.~\ref{fig:case_control}: STORM issues saturating steer reversals and braking spikes, whereas DynaDreamer commits to low-gain commands throughout, consistent with its lowest jerk reported in Fig.~\ref{fig:driving_metrics}. The root cause is representational: lacking an ego-dynamics prior, STORM cannot anticipate how its own steering warps the egocentric observation under the unfamiliar short-wheelbase chassis and resorts to over-correction. DynaDreamer's prior supplies this warp in advance, enabling the symmetric, low-gain tracking above.
 
\begin{figure}[!t]
\centering
\includegraphics[width=0.9\columnwidth]{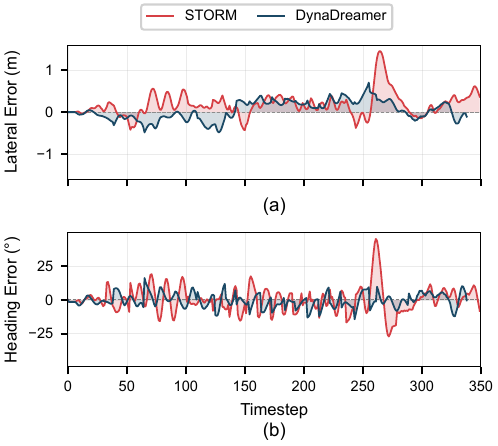}
\caption{Zero-shot Microlino tracking error over time: (a) lateral deviation and (b) heading error for STORM and DynaDreamer.}
\label{fig:case_traj}
\end{figure}
 
\begin{figure}[!t]
\centering
\includegraphics[width=0.9\columnwidth]{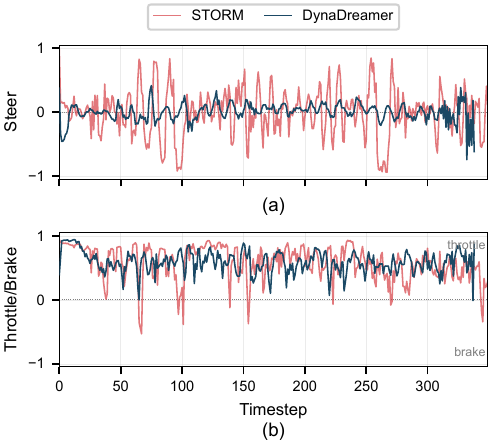}
\caption{Control inputs (STORM vs.\ DynaDreamer) on the zero-shot Microlino episode: (a) steering and (b) throttle and brake.}
\label{fig:case_control}
\end{figure}

\subsubsection{Low-Dimensional State Prediction and Parameter Identification}
 
DynaDreamer's smooth behavior stems from its Neural-ODE module. Fig.~\ref{fig:case_dyn} overlays all six predicted dynamic states against ground truth; predictions remain aligned over the entire episode, including high-frequency lateral and yaw content, underpinning the control quality above.
 
\begin{figure}[!t]
\centering
\includegraphics[width=0.9\columnwidth]{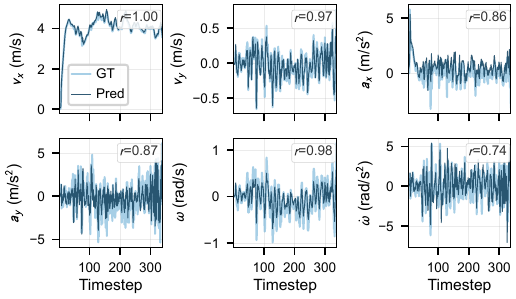}
\caption{DynaDreamer low-dimensional state prediction versus ground truth, with Pearson correlation $r$ annotated per state.}
\label{fig:case_dyn}
\end{figure}
 
Fig.~\ref{fig:case_param} traces zero-shot identification across both chassis extremes for a representative cross-section of the parameter vector: the two geometric terms $(l_f, l_r)$ and the rear cornering stiffness $C_r$. This triplet spans both the geometric and the lateral-force parameter groups and reflects the typical identification behavior. The two geometric terms settle near the ground truth with moderate steady-state bias, and $C_r$ tracks the ground-truth trend with a residual offset. Such residuals are expected because vehicle parameters are tightly coupled in the bicycle model; they do not impair control significantly, as the parameter head is auxiliary and the neural residual absorbs the bias, enabling generalization across the full chassis spectrum (Table~\ref{tab:param_summary}).
 
 
\begin{figure}[!t]
\centering
\includegraphics[width=0.9\columnwidth]{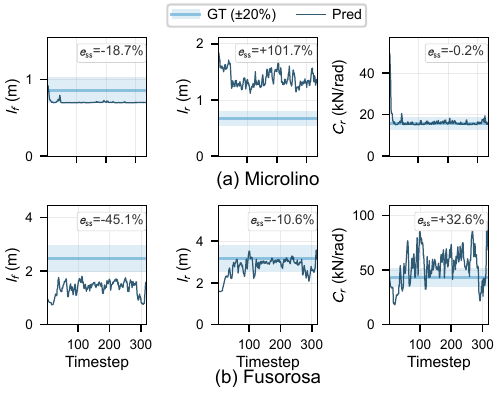}
\caption{Zero-shot identification of $(l_f, l_r, C_r)$ for (a) Microlino and (b) Fusorosa; blue band marks ground truth with $\pm 20\%$ tolerance, navy line is the running estimate, and $e_{\mathrm{ss}}$ is the steady-state relative error over the final $30\%$ of the episode.}
\label{fig:case_param}
\end{figure}
 
\section{Conclusion}
\label{sec:conclusion}
This paper formalizes the structural information bottleneck in BEV-based WM learning and introduces DynaDreamer to resolve it through explicit ego-dynamics conditioning grounded in identifiable chassis physics. By injecting a physics-informed context into the WM's prior and posterior distributions, the proposed method confines the WM to scene dynamics beyond ego-motion and achieves zero-shot cross-embodiment adaptation, with a single trained model serving a dynamically diverse fleet without per-platform retraining. More broadly, any observation-transition component that is exogenous and deterministic given a known context is a candidate for structural conditioning, and the principle may extend to other embodied agents whose own motion dominates the egocentric observation flow. Future work includes extending the ego-dynamics conditioning to multi-agent imagination, where each traffic participant carries its own dynamics context, and sim-to-real deployment in which each fleet platform contributes identification data and receives an adapted policy without retraining.

{\footnotesize
\bibliographystyle{IEEEtranN}
\bibliography{ref}
}


\end{document}